\pdfoutput=1

\documentclass{article}

\PassOptionsToPackage{numbers, compress}{natbib}

\usepackage[preprint]{neurips_2022}

\usepackage[utf8]{inputenc} 
\usepackage[T1]{fontenc}    
\usepackage{hyperref}       
\usepackage{url}            
\usepackage{booktabs}       
\usepackage{amsfonts}       
\usepackage{nicefrac}       
\usepackage{microtype}      
\usepackage{xcolor}         

\usepackage{microtype}
\usepackage{graphicx} 
\usepackage{grffile}
\usepackage{booktabs} 
\usepackage[most]{tcolorbox}

\usepackage{hyperref}

\usepackage{amsmath}
\usepackage{amssymb}
\usepackage{mathtools}
\usepackage{amsthm}

\usepackage[capitalize,noabbrev]{cleveref}

\theoremstyle{plain}
\newtheorem{assumption}{Assumption}

\usepackage{amsmath,amsfonts,bm}

\def\figref#1{figure~\ref{#1}}

\def\eqref#1{equation~\ref{#1}}

\def\algref#1{algorithm~\ref{#1}}

\def\1{\bm{1}}

\def\eps{{\epsilon}}

\def\vr{{\bm{r}}}

\def\vx{{\bm{x}}}
\def\vy{{\bm{y}}}
\def\vz{{\bm{z}}}
\def\vG{{\bm{G}}}

\DeclareMathAlphabet{\mathsfit}{\encodingdefault}{\sfdefault}{m}{sl}
\SetMathAlphabet{\mathsfit}{bold}{\encodingdefault}{\sfdefault}{bx}{n}

\newcommand{\E}{\mathbb{E}}

\newcommand{\R}{\mathbb{R}}

\newcommand{\Var}{\mathrm{Var}}

\usepackage{float}
\usepackage{mathtools}
\usepackage{framed}
\usepackage{pbox}
\usepackage{afterpage}
\usepackage{url}
\usepackage{array}
\usepackage{amsfonts}

\usepackage{multirow}
\usepackage{booktabs}
\usepackage{relsize}
\usepackage{xspace}
\usepackage{caption}
\usepackage{subcaption}
\newcommand{\ignore}[1]{}

\newcommand{\norm}[1]{\left\lVert#1\right\rVert}

\usepackage{framed}	
\usepackage{url}
\usepackage{enumerate}
\usepackage{color}

\usepackage{changepage}

\usepackage{algorithmic}
\usepackage{algorithm}

\usepackage{amsthm}
\theoremstyle{definition}
\theoremstyle{plain}

\newcommand{\SKIP}[1]{}

\newcommand{\bfQ}{\textbf{Q}}
\newcommand{\bftQ}{\tilde{\textbf{Q}}}

\renewcommand{\R}{\rm I\!R}

\newcommand{\calP}{\mathcal{P}}

\newcommand{\calD}{\mathcal{D}}
\newcommand{\calK}{\mathcal{K}}

\newcommand{\calU}{\mathcal{U}}

\newcommand{\bfw}{\mathbf{w}}

\renewcommand{\E}{\mathbb{E}}

\def\myref#1{{\color{red}{#1}}}%

\usepackage{xcolor}

\usepackage{bbm}

\graphicspath{{images/}}

\usepackage[page]{appendix}
\usepackage{chngcntr}
\usepackage{etoolbox}

\AtBeginEnvironment{subappendices}{%
\section*{Appendix}
\addcontentsline{toc}{section}{Appendix}
}

\usepackage{xspace}
\makeatletter
\DeclareRobustCommand\onedot{\futurelet\@let@token\@onedot}
\def\@onedot{\ifx\@let@token.\else.\null\fi\xspace}

\def\etal{\emph{et al}\onedot}
\makeatother

\usepackage{enumitem}

\def\figref#1{Fig.~\ref{#1}}

\def\eqref#1{Eq.~(\ref{#1})}

\def\algref#1{Algorithm~\ref{#1}}

\def\tabref#1{Table~\ref{#1}}

\def\tabref#1{Table~\ref{#1}}

\usepackage[acronym,smallcaps]{glossaries} 
\makeglossaries 
\glsdisablehyper
\newacronym{SGD}{sgd}{Stochastic Gradient Descent}
\newacronym{GD}{gd}{Gradient Descent}
\newacronym{PGD}{pgd}{Projected Gradient Descent}
\newacronym{DNN}{dnn}{Deep Neural Networks}
\newacronym{NN}{nn}{Neural Network}
\newacronym{BC}{bc}{BinaryConnect}
\newacronym{BWN}{bwn}{Binary Weight Network}
\newacronym{fc}{fc}{fully-connected}
\newacronym{REF}{ref}{Reference Network}
\newacronym{KL}{kl}{KL}
\newacronym{S}{s}{S}
\newacronym{LR}{lr}{LR}
\newacronym{KKT}{kkt}{KKT}
\newacronym{BOP}{bop}{Binary Optimizer}
\newacronym{STE}{ste}{Straight Through Estimator}
\newacronym{QN}{qn}{Quantization Networks}
\newacronym{QAT}{qat}{Quantization-Aware Training}
\newacronym{APQN}{apqn}{Additive Pseudo-Quantization Noise}
\newacronym{MQAT}{mqat}{Multi-bit Quantization-Aware Training}
\newacronym{FC}{fc}{FC}
\newacronym{BN}{bn}{BN}
\newacronym{KURE}{kure}{Kurtosis Regularization}
\newacronym{FL}{fl}{Federated Learning}

\newcommand{\fedavg}{\textsc{FedAvg}\xspace}

\newcommand{\slenet}[1]{{\small LeNet#1}}

\newcommand{\sresnet}[1]{{\small ResNet#1}}

\newcommand{\cifar}[1]{{\small CIFAR#1}}

\newcommand{\tinyimagenet}{{TinyImageNet}}

\newcommand{\femnist}[1]{{\small FEMNIST}}

\usepackage{pifont}

\usepackage{wrapfig}

\newif\ifsupp
\suppfalse

\newif\ifarxiv
\arxivfalse

\newif\iffinal
\finalfalse

\theoremstyle{plain}
\newtheorem{theorem}{Theorem}[section]

\newtheorem{lemma}[theorem]{Lemma}
\newtheorem{corollary}[theorem]{Corollary}
\theoremstyle{definition}

\theoremstyle{remark}

\usepackage[textsize=tiny]{todonotes}
\theoremstyle{plain}

\theoremstyle{definition}
\theoremstyle{remark}

\usepackage[textsize=tiny]{todonotes}

\title{Quantization Robust Federated Learning for Efficient Inference on Heterogeneous Devices}

\author{%
  Kartik Gupta$^\dagger$\thanks{Work done during internship at Qualcomm AI Research. Qualcomm AI Research is an initiative of Qualcomm Technologies, Inc. and/or its subsidiaries.}, 
  Marios Fournarakis$^\ddagger$, 
  Matthias Reisser$^\ddagger$, 
  Christos Louizos$^\ddagger$, 
  Markus Nagel$^\ddagger$ 
\\
$^\dagger$ Australian National University, $^\ddagger$ Qualcomm AI Research\\
$^\dagger$\texttt{kartik.gupta@anu.edu.au}\\
$^\ddagger$\texttt{\{mfournar,mreisser,clouizos,markusn\}@qti.qualcomm.com}
}

\begin{document}

\maketitle


\begin{abstract}
\acrfull{FL} is a machine learning paradigm to distributively learn machine learning models from decentralized data that remains on-device. Despite the success of standard Federated optimization methods, such as Federated Averaging (\fedavg) in \acrshort{FL}, the energy demands and hardware induced constraints for on-device learning have not been considered sufficiently in the literature. Specifically, an essential demand for on-device learning is to enable trained models to be quantized to various bit-widths based on the energy needs and heterogeneous hardware designs across the federation. In this work, we introduce multiple variants of federated averaging algorithm that train neural networks robust to quantization. Such networks can be quantized to various bit-widths with only limited reduction in full precision model accuracy. We perform extensive experiments on standard \acrshort{FL} benchmarks to evaluate our proposed \fedavg variants for quantization robustness and provide a convergence analysis for our Quantization-Aware variants in \acrshort{FL}. Our results demonstrate that integrating quantization robustness results in \acrshort{FL} models that are significantly more robust to different bit-widths during quantized on-device inference.
\end{abstract}
\section{Introduction}
\acrfull{FL} is a distributed machine learning paradigm, where a large number of clients, such as consumer smartphones, personal computers or smart home devices learn collaboratively. Clients train on their private local data, which is never shared with other participants in the federation, such as other clients or the server. Despite learning happening on-device, \acrshort{FL} results in a single global model at the end of training. The privacy of local client data is an important requirement in modern machine learning and acts as a central motivator for \acrshort{FL}. 

Several challenges arise in the \acrshort{FL} setting. For example, different clients might have different computational constraints based on their hardware design specifications. One practically relevant heterogeneous hardware characteristic is the supported quantization bit-width of the hardware accelerator. A suitably trained model should exhibit no significant performance degradation after quantization to various bit-widths represented in the heterogeneous device landscape. 

Recent works~\cite{kure,alizadeh2020gradient,Jangho2020PBGS} introduced novel ways to train quantization robust models in the centralized training setting, but the application of such quantization robustness mechanisms has not been a focus in \acrshort{FL} yet. In this work, we address the problem of learning quantization robust models trained using the standard federated learning algorithm, known as Federated Averaging (\fedavg)~\cite{mcmahan2017communication}. To this end, we introduce multiple variants of \fedavg algorithms to incorporate quantization robustness in \acrshort{FL}; this is done either via regularization-based methods for quantization robustness or modified quantization-aware training methods. \textit{Firstly}, we propose the integration of a standard quantization robust approach known as \acrfull{KURE}, which involves a regularization term in the local clients' loss function. \textit{Secondly}, we present a quantization robustness approach that involves adding random pseudo-quantization noise during the training procedure. This is motivated by the recent success of Randomized Smoothing~\cite{cohen2019certified} in the adversarial robustness literature that aims to learn models where the input data samples are corrupted with Gaussian noise. 
The adopted mechanism is also inspired by the introduction of additive pseudo-quantization noise~\cite{defossez2021differentiable} for \acrshort{QAT} that discards the \acrfull{STE} approximation for non-differentiable uniform quantization.

\acrfull{QAT} methods \cite{dorefa, krishnamoorthi,esser2021lsq, nagel2021white} have been successful at training quantized models with ultra-low bit-widths. \acrshort{QAT} models perform very well for the target bit-width they have been trained on but can lead to significant degradation for other bit-widths, even for full precision \citep{Jangho2020PBGS}. 
To address this limitation of conventional \acrshort{QAT} methods we introduce \acrfull{MQAT}, a novel \acrshort{QAT} framework that achieves quantized models robust to multiple bit-widths without re-training.

In \acrshort{MQAT} for \acrshort{FL}, a random bit-width is sampled for each client from the set of considered quantization bit-widths before performing a standard \acrshort{QAT} procedure during the client training phase. This small modification enables models trained using the federated regime to be robust to multiple bit-widths during quantized inference. Furthermore, since \acrshort{QAT} involves certain heuristics (\acrshort{STE}) for computing gradients (due to the non-differentiable rounding operation), we theoretically analyse the convergence behaviour of the global model in the non-convex setting when clients perform local \acrshort{QAT}.

In summary, we make the following contributions in this paper:
\begin{itemize}[leftmargin=*]
    \item We introduce multiple quantization robustness methods such as \acrfull{KURE}, \acrfull{APQN} into the federated learning setup to achieve quantization robust models that can be used for efficient inference at multiple bit-widths.
    \item As the standard form of \acrshort{QAT} integration into federated learning fails to generalise across multiple bit-width, we propose \acrfull{MQAT}, a new variant to achieve quantization robust models in \acrshort{FL}. 
    \item We study the theoretical convergence properties for our proposed \acrshort{QAT} variants in federated learning and show that the convergence rate for the server-side weights is similar to traditional \acrshort{FL}, albeit with a \acrshort{QAT}-method-specific error floor.      
    \item We perform extensive experimental evaluations of baselines and our \fedavg variants on \cifar{-10}, \cifar{-100}, \femnist{} and \tinyimagenet{} with different network architectures. We empirically demonstrate that our proposed modifications yield models that are robust to quantization at multiple bit-widths without significant reduction on the model's full-precision accuracy.
\end{itemize}


\section{Preliminaries}
We first provide some brief background on \fedavg and neural network quantization robustness.
\subsection{Federated Averaging}
Federated learning is a distributed learning paradigm where multiple clients collaboratively learn a shared model. In this machine learning framework, the local client data is not shared with other clients or the server.  
The problem of federated learning can be formulated as an optimization objective
\begin{align}
    \min_{\bfw} F(\bfw) &= \E_{i \sim \calP}[ F_i(\bfw)],
    \label{eqn:global_obj}
\end{align}
where $F_i(\bfw, \calD_i) = \E_{\xi \sim \calD_i}[f_i(\bfw, \xi)]$ is the the local objective function at client $i$, $\bfw \in \mathbb{R}^D$ represents the parameters for the global model, and $\calP$ denotes a distribution over the population of clients $\mathcal{I}$.  The local data distribution $\calD_i$ often varies between clients, resulting in data heterogeneity.

Federated Averaging (\fedavg)~\cite{mcmahan2017communication} is the standard algorithm for federated learning. It operates via a series of \emph{rounds} where each round is divided into a client update phase and server update phase. We denote the total number of rounds as $T$. At the beginning of each round $t$, a subset of clients $\mathcal{S}_t$ is sampled from the pool of clients. 
The server model is then shared with the sampled clients. During the client update phase, each sampled client runs local \acrshort{SGD} for $K$ steps with learning rate $\eta_c$ using their own private data. We denote the $i$-th client's parameters at the $k$-th local step of the $t$-th round by $\bfw_{t,k}^i$.
During the server update phase, the updates of the sampled clients are averaged to build the server-side update $\Delta_t$. The server then applies that update with learning rate $\eta_s$ to receive the next round's parameters $\bfw_{t+1}$. Reddi \etal~\cite{reddi2020adaptive} describe a generalization of the server-side update rule to include more advanced adaptive optimizers. 

\subsection{Quantization Robustness}
\label{sec:prelim_quant}
The objective for robust quantization is to learn a single model that can be quantized to different bit-widths without significant degradation in the full precision performance. Given a neural network parameterized by $\bfw$ that is  optimized using a standard loss function $F$, such as the cross entropy, quantization robustness aims at minimizing the following loss:
\begin{align}
    \min_{\bfw} & \quad \sum_{b \in B} F(\bfQ(\bfw, b), \calD).     
\label{eqn:quant_robust}
\end{align}
Here,quantizer $\bfQ(\cdot)$ with bit-width $b$ and a quantization step size $\Delta_b$, we have that 
%
%

\begin{equation}
    \bfQ(\bfw, b) = \Delta_b \cdot \text{clip} \left( \left\lfloor \frac { \bfw } { \Delta_b } \right\rceil,   -2^{b-1}, 2^{b-1}-1 \right), 
\end{equation}

where $\left\lfloor \cdot \right\rceil$ denotes the rounding-to-nearest integer operation, and $\text{clip}(\cdot)$ clamps its input such that it lies in the range$[-2^{b-1}, 2^{b-1}-1]$. The quantization step size can be estimated either post-training or learnt using \acrshort{QAT}~\cite{krishnamoorthi, esser2021lsq, nagel2021white}. The above objective intends to learn a neural network that is robust to various bit-widths in the quantization bit set $B$. $B$ could also include the high precision 32-bit floating-point format (FP32). Note that the above formulation explicitly enforces robustness to different bit-widths for weight quantization only. It is straightforward to enforce quantization robustness for activations in a similar manner. Recent works~\cite{alizadeh2020gradient, kure} have explored alternate ways of satisfying the above objective by adding a regularization term in the standard training procedure instead of directly solving the aforementioned optimization problem. 

\SKIP{

Based on the first order taylor explansion of the quantized model, a regularization approach~\cite{alizadeh2020gradient} was proposed to minimize the L1 norm of the gradients with respect to weight and activations. It was motivated by the fact the L1 norm of gradients is equivalent to maximum bound on the quantization noise.

following term:
\begin{align}
    L_\text{grad} =  \norm{\nabla_{\bfw} L(\bfw; \calD)}_1,
\label{eqn:grad_l1}
\end{align}
where $\norm{\cdot}_1$ of weight gradients is shown to be the maximum bound on the quantization noise.

Since the gradient regularization is computationally expensive due to second-order gradient information, a simpler alternative regularization procedure was proposed in follow-up work~\cite{kure}. The authors have shown that uniform distributions on the weight tensors achieve better quantization robustness instead of normal distributions attained during standard training procedure. The authors propose to enforce a uniform distribution on the weight tensors of the neural network using kurtosis as a proxy for shape of distribution.

}
\section{A Federated Learning Framework with Quantization Robustness}
In quantization robust federated learning we aim to solve an optimization problem of the following form:
\begin{align}
    \min_{\bfw} F(\bfw) &= \E_{i \sim \calP}[ F^*_i(\bfw)], \label{eqn:global_quantrobustness_obj}
\end{align}
where $F^*_i(\bfw, \calD_i) = \E_{\xi \sim \calD_i} \sum_{b \in B} [f_i(\bfQ(\bfw, b), \xi)]$ is a modified local loss to encourage quantization robustness at client $i$ and $B$ is the set of target quantization bit-widths. Instead of directly optimizing this loss, which involves multiple forward-backward passes through the same batch for each  bit-width, we incorporate various, more efficient ways for quantization robustness in the \fedavg framework.

\paragraph{Regularization Method.} Regularization methods such as Kurtosis regularization~\cite{kure}, which enforce a uniform distribution on the weight tensors, can be incorporated in the \fedavg framework by modifying the loss function $F_i$ for each client  as follows:
\begin{align}
F^*_i(\bfw, \calD_i) = \E_{\xi \sim \calD_i}[f_i(\bfw, \xi)] + \lambda \cdot L_\text{KURE}(\bfw).
\end{align}
The Kurtosis regularization term for an $M$- layered neural network can be expressed as
\begin{align}
    L_\text{KURE} = \frac{1}{M} \sum_{i=1}^M | \calK(\bfw_i) - \calK_\tau|^2, \qquad \calK(\bfw) = \E\bigg[\bigg(\frac{\bfw - \mu}{\sigma}\bigg)^4 \bigg].
\label{eqn:kure}
\end{align}
Here, $\mu$ and $\sigma$ are the mean and standard deviation of tensor $\bfw$ and $\calK_\tau$ denotes a scalar value that defines the distribution enforced on the tensors. Similar to \cite{kure}, we employ $\calK_\tau=1.8$ to enforce uniform distribution. We provide the pseudo-code for \fedavg with Kurtosis regularization, \fedavg-\acrshort{KURE}, in \algref{alg:fedavg}.

\paragraph{\acrfull{APQN}.} The quantization robustness problem has similarities with adversarial robustness in the sense that both aim to keep predictions unaltered in the presence of some form of bounded additive noise. Adversarially robust models aim to be robust to noised-up input, whereas with quantization robustness the noise is added to either the weight tensor or the intermediate activations. We present a quantization robustness approach that involves adding random pseudo-quantization noise during the training procedure. This is motivated by the recent success of Randomized Smoothing~\cite{cohen2019certified} in the adversarial robustness literature, where a model is learnt with input data samples corrupted with Gaussian noise. An adaptation similar to ours has been presented in the recently introduced Differentiable Quantizer~\cite{defossez2021differentiable} as a replacement of the commonly used \acrfull{STE} based quantizer. Their proposed quantizer has only been used to achieve quantized models for a single target bit-width or mixed-precision with fixed computational budget. Since \acrshort{APQN} involves additive pseudo quantization noise and the noise does not have any learnable parameters, backpropagation for the rest of the parameters is straightforward.

We propose to learn models that are robust to varying levels of quantization noise and thus can be quantized to different bit-widths. The local loss function in this case can be reformulated as
\begin{align}
F^*_i(\bfw, \calD_i) &= \E_{\xi \sim \calD_i} [f_i(\bftQ(\bfw, b), \xi)].
\label{eqn:diffq_obj}
\end{align}
Here, $\bftQ(\cdot)$ is a pseudo-quantizer with bit-width $b$ that adds noise sampled from the uniform distribution $\calU[-\Delta_b/2, \Delta_b/2]$, and can be defined as
\begin{align}
    \bftQ(\bfw, b) &= \bfw +
    \calU\bigg[-\frac{\Delta_b}{2},\frac{\Delta_b}{2}\bigg].
\end{align}
Provided that $\Delta_b$ is sufficiently large, the sampled pseudo-quantization noise can have support for various target bit-widths, thus encouraging quantization robustness. 
The pseudo-code for \fedavg-\acrshort{APQN} is provided in \algref{alg:fedavg-qat}.




\paragraph{\acrshort{QAT} and \acrshort{MQAT}.} A standard way of learning a network with low bit-widths is to constrain the parameters and/or activations of the model to fixed quantization levels. This procedure of training a neural network with standard \acrfull{PGD} algorithm with quantization function as projection, is often termed \acrfull{QAT}. To perform ``quantization-aware'' \acrshort{FL}, we can adopt the \acrshort{QAT} procedure for the local optimization at each client to learn a global model that can be quantized to a specific bit-width. In this ``quantization-aware'' \acrshort{FL}, the client-level loss function can be reformulated as:
\begin{align}
F^*_i(\bfw, \calD_i) &= \E_{\xi \sim \calD_i} [f_i(\bfQ(\bfw, b), \xi)].
\label{eqn:qat_fl}
\end{align}
%
%
%
The quantization step-size $\Delta_b$ can be either learnt as a parameter~\cite{esser2021lsq} or be estimated before the start of training and kept fixed thereafter. Note that the backward pass of the network involves a gradient estimate through the non-differentiable rounding operation of $\bfQ(\cdot)$. To this end, similar to previous \acrshort{QAT} literature we use the standard \acrshort{STE} approximation~\cite{bengio2013estimating}, which approximates the gradient of the rounding operator as $1$:
\begin{align}
     \frac{\partial \left\lfloor x \right\rceil}{\partial x} = 1.
\end{align}
By using this \acrshort{STE} approximation for \acrshort{QAT}, we can train models that can be quantized to specific bit-widths. In order to use the trained model for bit-widths other than $b$, we can analytically estimate the quantization step-size using the ranges for bit-width $b$ as follows:
\begin{align}
     \Delta_a = \frac{2^b - 1}{2^a - 1}\Delta_b.
\end{align}
Here, $a$ refers to any target bit-width during inference stage and $b$ refers to bit-width for which the model is trained on. 
The pseudo-code for \fedavg-\acrshort{QAT} is provided in \algref{alg:fedavg-qat}.

Although \acrshort{QAT} trains models that perform favourably at specific bit-widths, they often suffer from performance degradation when quantizing to other bit-widths~\cite{Jangho2020PBGS}. For this reason, \acrshort{QAT} alone is not suitable for tackling the heterogeneous hardware requirements that one can encounter in a cross-device \acrshort{FL} deployment. 
We propose \acrshort{MQAT} to realize quantization-robust \acrshort{FL}. \acrshort{MQAT} aims to learn models that are robust to a set of bit-widths $B$, by selecting a bitwidth $b \in B$ either randomly during local training or by fixing it on the basis of each client's hardware requirements. The sampled bit-width is then kept same for all the layers. 

Similar to \acrshort{QAT}, the quantization step-size $\Delta_b$ for different bit-widths can be either learnt or estimated before the start of training and kept fixed thereafter. The quantization step-size for different bit-widths is then shared along the model parameters with all clients. 
We provide the pseudo-code for \fedavg-\acrshort{MQAT} in \algref{alg:fedavg-qat}.

\noindent\begin{minipage}{\textwidth}
  \centering
\begin{minipage}{0.49\textwidth} \vspace{-.26cm}
\begin{algorithm}[H]
  \caption{\fboxsep2pt \colorbox{red!30}{\fedavg}, \fboxsep2pt \colorbox{green!30}{\fedavg-\acrshort{KURE}}}
  \label{alg:fedavg}
  \begin{algorithmic}[1]
  \STATE {\textbf{Require}} ($\bfw_0, \eta_c, \eta_s, \lambda$)
  \FOR {$t=0, \ldots, T-1$}
    \STATE sample a subset $\mathcal{S}$ of clients 
      \FORALL {\kern-6.5em\kern+6.5em $i \in \mathcal{S}$ {\bf in parallel}}
        \STATE $\bfw_{t,0}^i \gets \bfw_{t}$ \COMMENT{broadcast server state to client}
        \FOR {$k = 0, \ldots, K-1$}
          \STATE \colorbox{red!30}{$g_{t,k}^i \gets \nabla f_i(\bfw_{t,k}^i; \xi_{t,k}^m)$} 
            \STATE \colorbox{green!30}{$\begin{aligned}f^*_i(\bfw_{ t,k}^i;\xi_{t,k}^m) &= f_i(\bfw_{ t,k}^i; \xi_{t,k}^m) \\ &+ \lambda \cdot L_\text{KURE}(\bfw_{t,k}^i) \end{aligned}$} 
          \STATE \colorbox{green!30}{$g_{t,k}^i \gets \nabla f^*_i(\bfw_{t,k}^i; \xi_{t,k}^m)$}         
          \STATE $\bfw_{t,k+1}^i \gets \bfw_{t,k}^m - \eta_c \cdot g_{t,k}^i$ 
          \COMMENT{client update} 
        \ENDFOR
      \ENDFOR
  \STATE $\Delta_t = \frac{1}{|\mathcal{S}|} \sum_{i \in \mathcal{S}} (\bfw_{t, K}^i - \bfw_{t, 0}^i)$ 
  \STATE $\bfw_{t+1} \gets \bfw_{t} + \eta_s \cdot \Delta_t$ \COMMENT{server update}
  \ENDFOR
\end{algorithmic}
\end{algorithm}
\end{minipage}
\hfill
\begin{minipage}{0.5\textwidth}

\begin{algorithm}[H]
  \caption{\fboxsep2pt \colorbox{green!30}{\fedavg-\acrshort{APQN}}, \fboxsep2pt \colorbox{red!30}{\fedavg-\acrshort{QAT}}, and \fboxsep2pt \colorbox{blue!20}{\fedavg-\acrshort{MQAT}}}
  \label{alg:fedavg-qat}
  \begin{algorithmic}[1]
  \STATE {\textbf{Require}} ($\bfw_0, \eta_c, \eta_s, b, B$)
  \FOR {$t=0, \ldots, T-1$}
    \STATE sample a subset $\mathcal{S}$ of clients 
      \FORALL {\kern-6.5em\kern+6.5em $i \in \mathcal{S}$ {\bf in parallel}}
        \STATE $\bfw_{t,0}^i \gets \bfw_{t}$ \COMMENT{broadcast server state to client}
        \STATE \colorbox{blue!20}{$b^\prime \gets \calU[B]$} 
        \FOR {$k = 0, \ldots, K-1$}
          \STATE \colorbox{green!30}{$g_{t,k}^i \gets \nabla_\bfw f_i(\bftQ(\bfw_{t,k}^i, b); \xi_{t,k}^m)$} 
          \STATE \colorbox{red!30}{$g_{t,k}^i \gets \nabla_\bfw f_i(\bfQ(\bfw_{t,k}^i, b); \xi_{t,k}^m)$} 
          \STATE \colorbox{blue!20}{$g_{t,k}^i \gets \nabla_\bfw f_i(\bfQ(\bfw_{t,k}^i, b^\prime); \xi_{t,k}^m)$} 
          \STATE $\bfw_{t,k+1}^i \gets \bfw_{t,k}^m - \eta_c \cdot g_{t,k}^i$ 
          \COMMENT{client update} 
        \ENDFOR
      \ENDFOR
  \STATE $\Delta_t = \frac{1}{|\mathcal{S}|} \sum_{i \in \mathcal{S}} (\bfw_{t, K}^i - \bfw_{t, 0}^i)$ 
  \STATE $\bfw_{t+1} \gets \bfw_{t} + \eta_s \cdot \Delta_t$ \COMMENT{server update}
  \ENDFOR
\end{algorithmic}
\end{algorithm}
\end{minipage}
\end{minipage}

\paragraph{Convergence analysis.} \acrshort{APQN}, \acrshort{QAT} and \acrshort{MQAT} modify the forward pass of the model, either by adding uniform noise in the case of \acrshort{APQN} or approximating the gradient of the non-differentiable rounding operation in the case of \acrshort{QAT} and \acrshort{MQAT}. Therefore, it is important to study how these modifications / heuristics affect the convergence behaviour of the global model in the federated setting. What follows now is a convergence analysis in the non-convex setting with the help of the following assumptions:
\begin{assumption}[Lipschitz Gradient]\label{as:lipschitz}
Each local loss function $F_s$ is $L$-smooth $\forall s \in \mathcal{S}$, \emph{i.e.}, $\|\nabla F_s(\vx) - \nabla F_s(\vy)\| \leq L\|\vx - \vy\|$, $\forall \vx, \vy \in \R^D$.
\end{assumption}
\begin{assumption}[Bounded variance] \label{as:boundedVariance}
Each $F_s$ has bounded local variance, \emph{i.e.}, $\E[\|\nabla f_s(\bfw, \epsilon) - \nabla F_s(\bfw)\|^2] \leq \sigma^2_l$, where $f_s$ is a stochastic estimate of the local loss based on a $\bfw \in \R^D$ and $\epsilon$ is a random mini-batch. Furthermore, the global variance is also bounded, \emph{i.e.}, $\frac{1}{S}\sum_s \|\nabla F_s(\bfw) - \nabla F(\bfw)\|^2 \leq \sigma^2_g$, $\forall \bfw \in \R^D$.
\end{assumption}
\begin{assumption}[Bounded quantization noise]\label{as:boundedQuantNoise}
Let $\bfw$ be a variable to be quantized, $j$ be any of its dimensions and $Q(\cdot)$ the quantization operation. We assume that the noise added in order to quantize $w_j$, \emph{i.e.}, $r_j = Q(w_j) - w_j$ is bounded by half the step size of the quantizer, \emph{i.e.}, $\frac{\Delta}{2}$.
\end{assumption}
The first two assumptions are common in the non-convex optimization literature for \acrshort{FL}~\cite{reddi2020adaptive} and the third is automatically satisfied in our quantization schemes, provided the ranges are set up appropriately. Based on these assumptions, we can then make the following statement.
\begin{theorem}\label{thm:conv_proof}
Let $K$ be the local iterations of each client and $\bfw \in \R^D$ be the global model parameter vector. Under assumptions~\ref{as:lipschitz},~\ref{as:boundedVariance} and ~\ref{as:boundedQuantNoise}, if the client ($\eta_c$) and server ($\eta_s$) learning rates are chosen such that 
\begin{align}
    \eta_c \leq \frac{1}{10 L K}, \qquad 
    \eta_c \leq \frac{1}{8LK\eta_s},
\end{align}
we have that the \fedavg-\{\acrshort{APQN},\acrshort{QAT},\acrshort{MQAT}\} server updates satisfy
\begin{align}
    & \min_{1\leq t \leq T} \|\nabla F(\bfw_t)\|^2  \leq \frac{F(\bfw_1) - F(\bfw^*)}{T\eta_s\eta_c A} + \frac{\eta_c}{\eta_s A}(B\sigma^2_l + \Gamma K \sigma_g^2 + H L^2D R^2),
\end{align}
where we define
\begin{align}
    A = \frac{K}{4} - 2L\eta_s \eta_c K^2, & \qquad 
    B = 4\eta_s \eta_cK^2L^2 + L\eta_s^2\left(2K^2 + \frac{K}{6}\right),\\
    \Gamma = 24\eta_s \eta_c K^2 L^2 + L\eta_s^2K, & \qquad
    H = \frac{4\eta_s}{3\eta_c}K + 6L\eta_s^2K^2,
\end{align}
and $R = \frac{\Delta_b}{\sqrt{12}}$ for \acrshort{APQN}, $R = \frac{\Delta_b}{2}$ for \acrshort{QAT} and $R = \max_b \frac{\Delta_b}{2}$ for \acrshort{MQAT}.
\end{theorem}
The proof is delegated to appendix due to space constraints. It follows~\cite{reddi2020adaptive} while handling quantization noise, such as~\cite{li2017training}. We can see that the convergence rate for the (non-quantized) server side weights is similar to traditional \acrshort{FL}~\cite{reddi2020adaptive}, albeit with an additional error floor due to the quantization noise. Some of this quantization error can be reduced by decreasing the learning rates, up to an irreducible factor of $\mathcal{O}(L^2DR^2)$. 

\section{Related Work}

\paragraph{Quantization-aware training.} \acrshort{QAT} is one of the more effective and widely used methods for achieving low-bit weight and activation quantization. It relies on simulating the quantization operation operation during training and requires access to labelled training data. Esser \etal~\cite{esser2021lsq} first introduced the idea of learning the quantization step-size jointly with the weight achieving near floating-point accuracy in ResNets even with $3$-bit quantization. Since then, further advances in quantization-aware training \citep{Han2021binreg, Lee2021EWGS, Gong2019DSQ,lsq+, Park2020profit} have pushed the envelop and enabled ultra low-bit quantization ($2$-$4$ bits) for a wide range of networks and tasks. 

\paragraph{Quantization robustness.} A drawback of \acrshort{QAT} is that it can make the trained model highly dependent on the chosen bit-width and quantization parameters. To address this, robust quantization aims at training a single set of weights that are robust to a wider range of quantization choices and bit-widths. Alizadeh \etal~\cite{alizadeh2020gradient} model quantization noise as an additive perturbation and show that they can improve quantization robustness by regularizing the $L_1$-norm of gradients. 
Since this type of gradient regularization is computationally expensive due to the second-order gradient information, a simpler alternative regularization procedure was proposed in follow-up work~\cite{kure}. Shkolnik \etal~\cite{kure} trains the network with \acrfull{KURE} on the weights to improve robustness and use the LAPQ \citep{nahshan2020LAPQ} algorithm to find the optimal quantization parameters post-training. It has been shown in~\cite{kure} that the uniform distributions on the weight tensors achieve better quantization robustness, instead of the Gaussian-like distributions attained during the standard training procedure. 
Recently, Defossez \etal~\cite{defossez2021differentiable} introduce additive uniform noise to simulate quantization during training for the purpose of \acrshort{QAT}; this method can be easily extended for the purpose of quantization robustness as we show in this paper. Jangho \etal~\cite{Jangho2020PBGS} introduces a training method that leads to more ``quantization-friendly'' weights, by scaling the gradient depending on the distance of the weights from the quantization grid. 

\paragraph{Quantization in \acrfull{FL}.}
Quantization in the context of \acrshort{FL} has been studied mostly in the context of compressing communication, \emph{e.g.},~\cite{amiri2020federated,reisizadeh2020fedpaq,triastcyn2021dp}. Equally important to reducing the communication overhead is the reduction of the computational cost of training and inference on-device. Existing work in this direction can roughly be divided into designing more efficient models through \emph{e.g.} sparsification~\cite{caldas2018expanding,jiang2019model,louizos2021expectation} and more effective algorithms \cite{reddi2020adaptive,he2020group} that reduce the number of training rounds.
Structurally sparse models can reduce training and inference costs (\cite{horvath2021fjord} targets heterogeneous computational resources) and may additionally reduce communication \cite{louizos2021expectation}. Quantization, the focus of this work, is orthogonal to these and as of yet understudied in the literature. To the best of our knowledge, this work is the first to investigate the robustness of models to different levels of quantization in the \acrshort{FL} setting.

\section{Experiments}
We evaluate the quantization robustness of different proposed variants of \fedavg using standard benchmarks in \acrshort{FL}. In this work, we mainly focus on weight quantization robustness of various trained models but we also provide additional experimental comparisons for activation quantization and quantizing both weight and activations. For all the results in the paper, we present the accuracy plots at different quantization bit-widths and we refer the reader for exact numbers to Appendix. 

\paragraph{Experimental Setup.} For the experimental comparisons, we use federated versions of the \cifar{-10}, \cifar{-100}~\cite{krizhevsky2009learning}, \tinyimagenet{}\footnote{\url{https://tiny-imagenet.herokuapp.com/}} and \femnist{}~\cite{caldas2018leaf} datasets. We use \slenet{-5} and \sresnet{-20} architecture for evaluation on \cifar{-10} and \slenet{-5} on \femnist{}. For experiments with \cifar{-100} and \tinyimagenet{}, we use \sresnet{-20} and \sresnet{-18} network architectures respectively. We replace batch-norm with group-norm in \sresnet{-20} since batch-norm induces training instabilities in \acrshort{FL}~\cite{reddi2020adaptive}. We report the accuracy of the quantized server-side model whenever comparing the models at different quantization bit-widths. To compare quantization robustness, we use bit-widths from the set $\{32, 8, 6, 4, 3, 2\}$.  We provide the details regarding training hyperparameters and datasets in Appendix. For the \acrshort{APQN} and \acrshort{KURE} methods, we employ post-training quantization methods to set the quantization ranges. Specifically, we find ranges that minimize the mean-squared error~\cite{nagel2021white}.
For \acrshort{QAT} based methods, range estimation is done at the start of the training and then the same ranges are used at inference time. Our code is written in PyTorch and all our experiments are performed using NVIDIA Tesla V100 GPUs. 

\begin{figure*}[t]
     \centering
     \begin{subfigure}[b]{0.45\textwidth}
         \centering
         \includegraphics[width=0.99\textwidth]{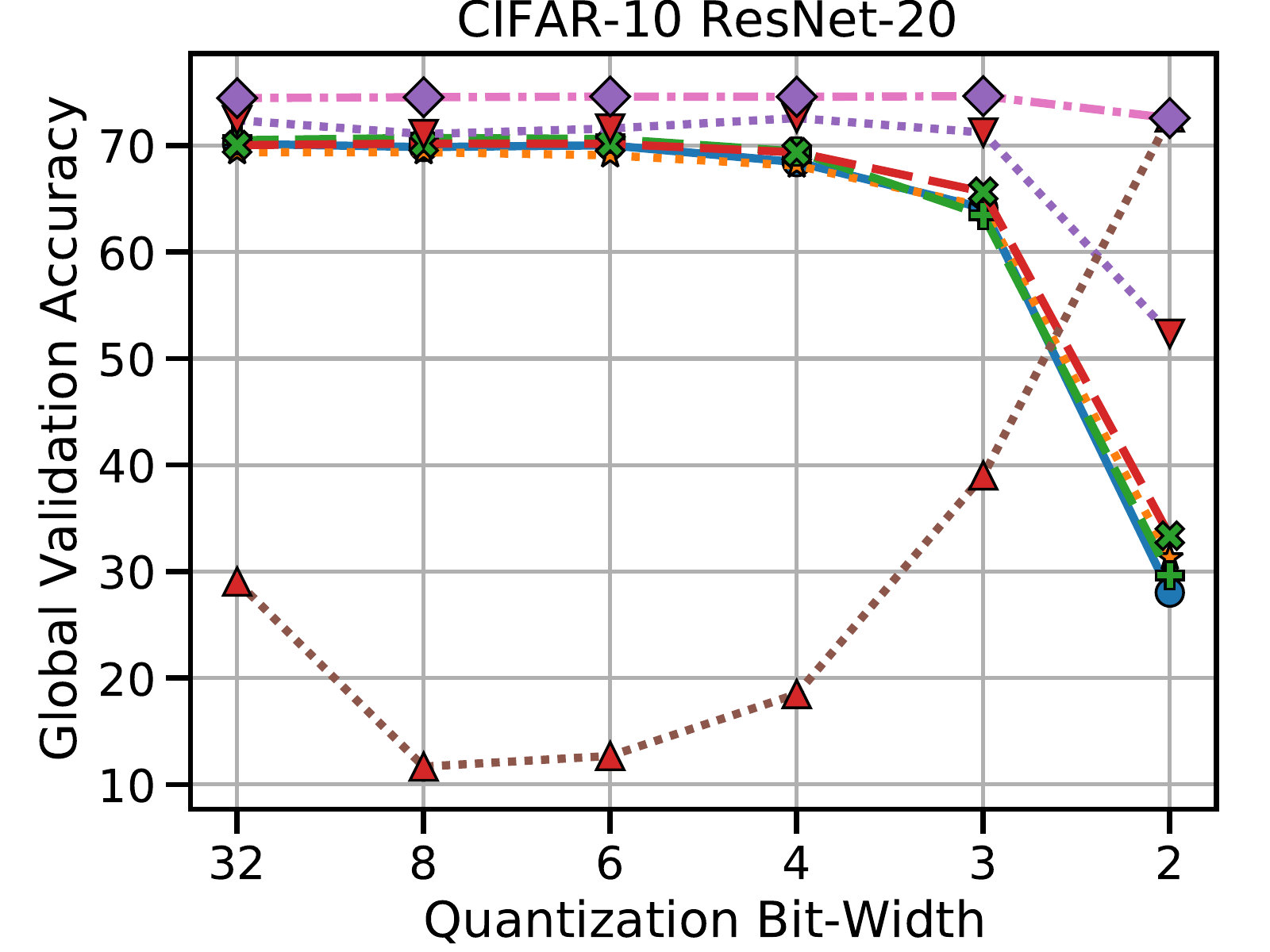}
         \caption{}
         \label{fig:cifar10resnet20_wqr}
     \end{subfigure}
     \begin{subfigure}[b]{0.45\textwidth}
         \centering
         \includegraphics[width=0.99\textwidth]{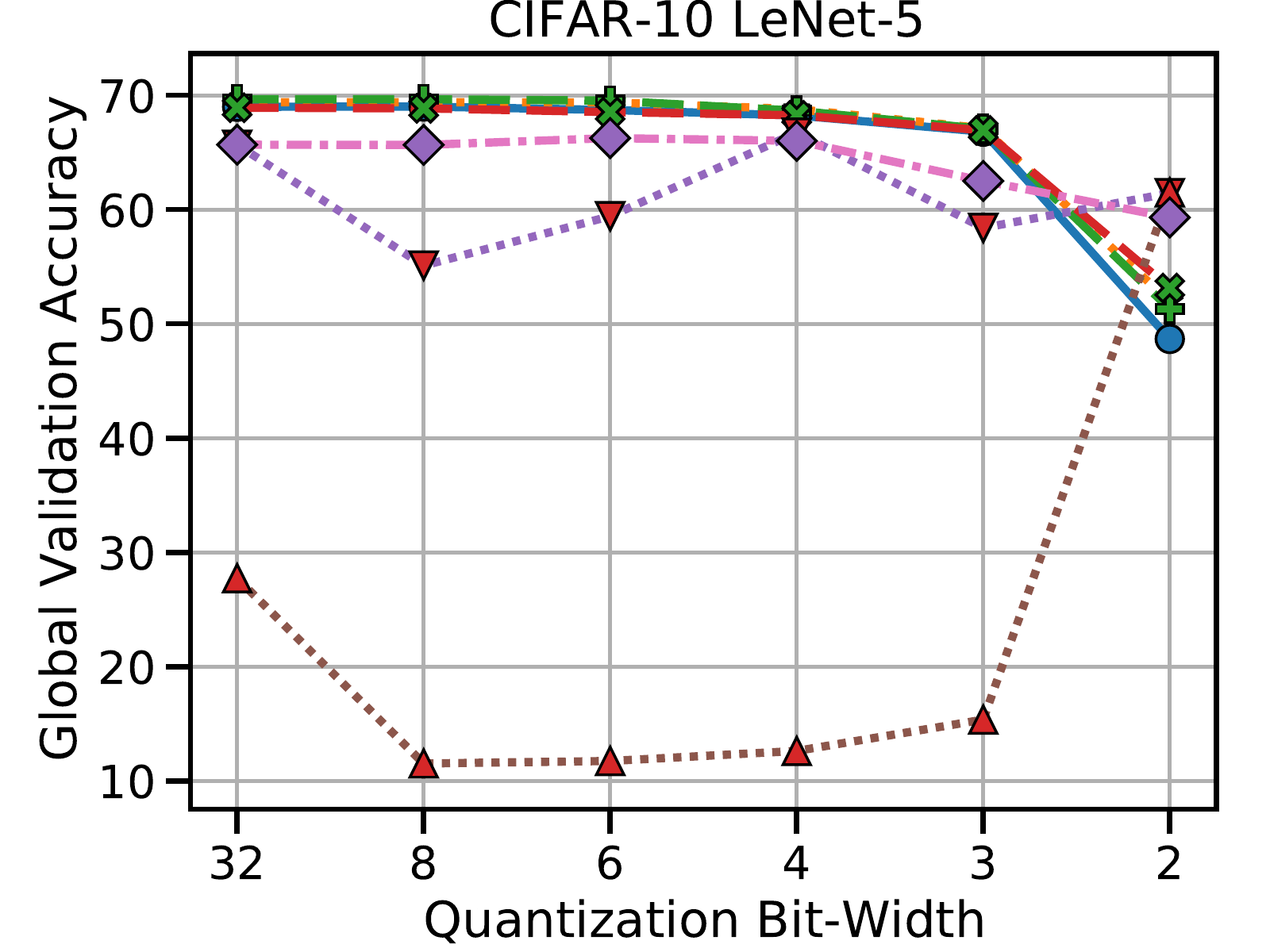}
         \caption{}
         \label{fig:cifar10lenet5_wqr}
     \end{subfigure}
     \begin{subfigure}[b]{0.45\textwidth}
         \centering
         \includegraphics[width=0.99\textwidth]{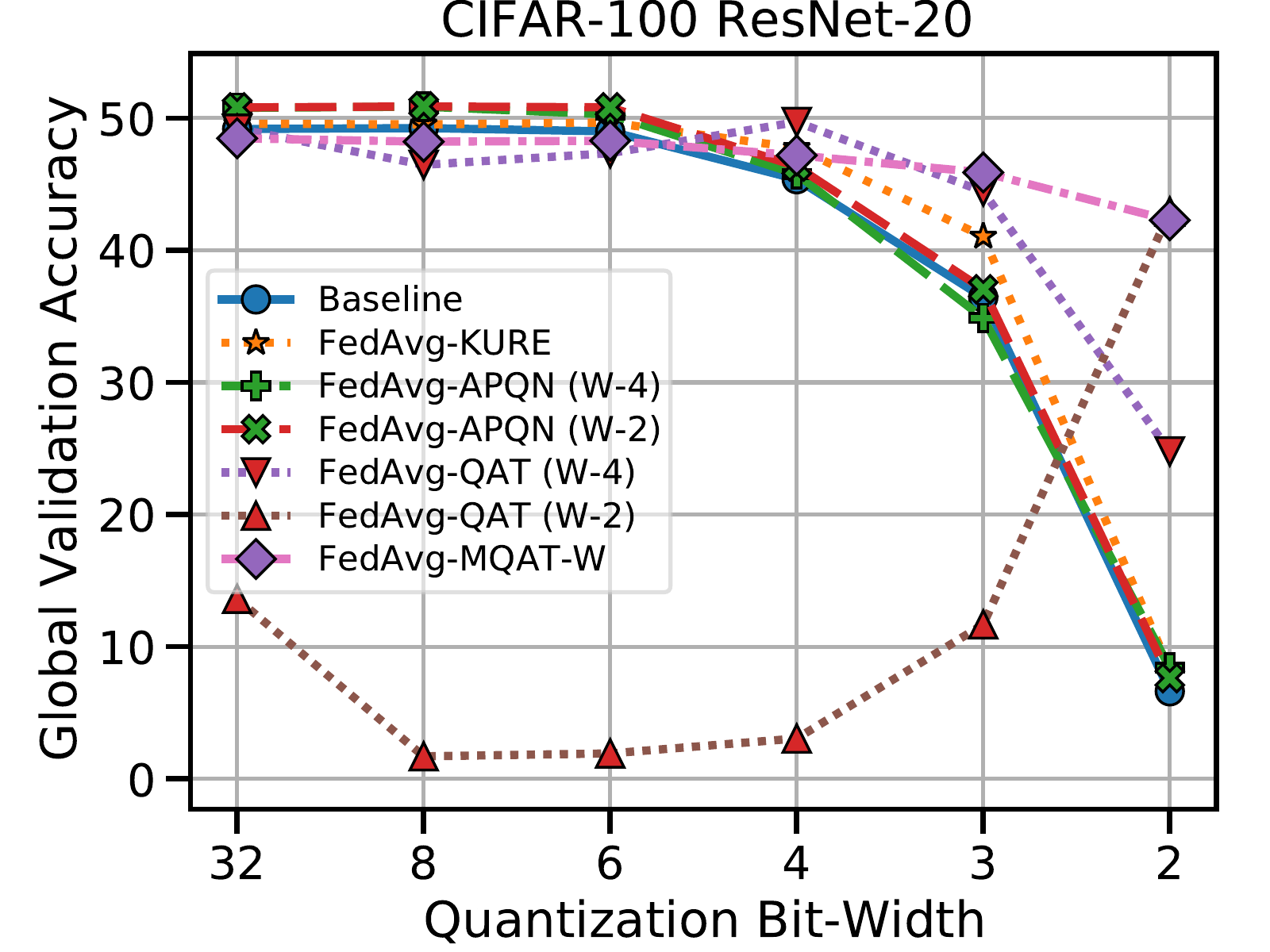}
         \caption{}
         \label{fig:cifar100resnet20_wqr}
     \end{subfigure}
     \begin{subfigure}[b]{0.45\textwidth}
         \centering
         \includegraphics[width=0.99\textwidth]{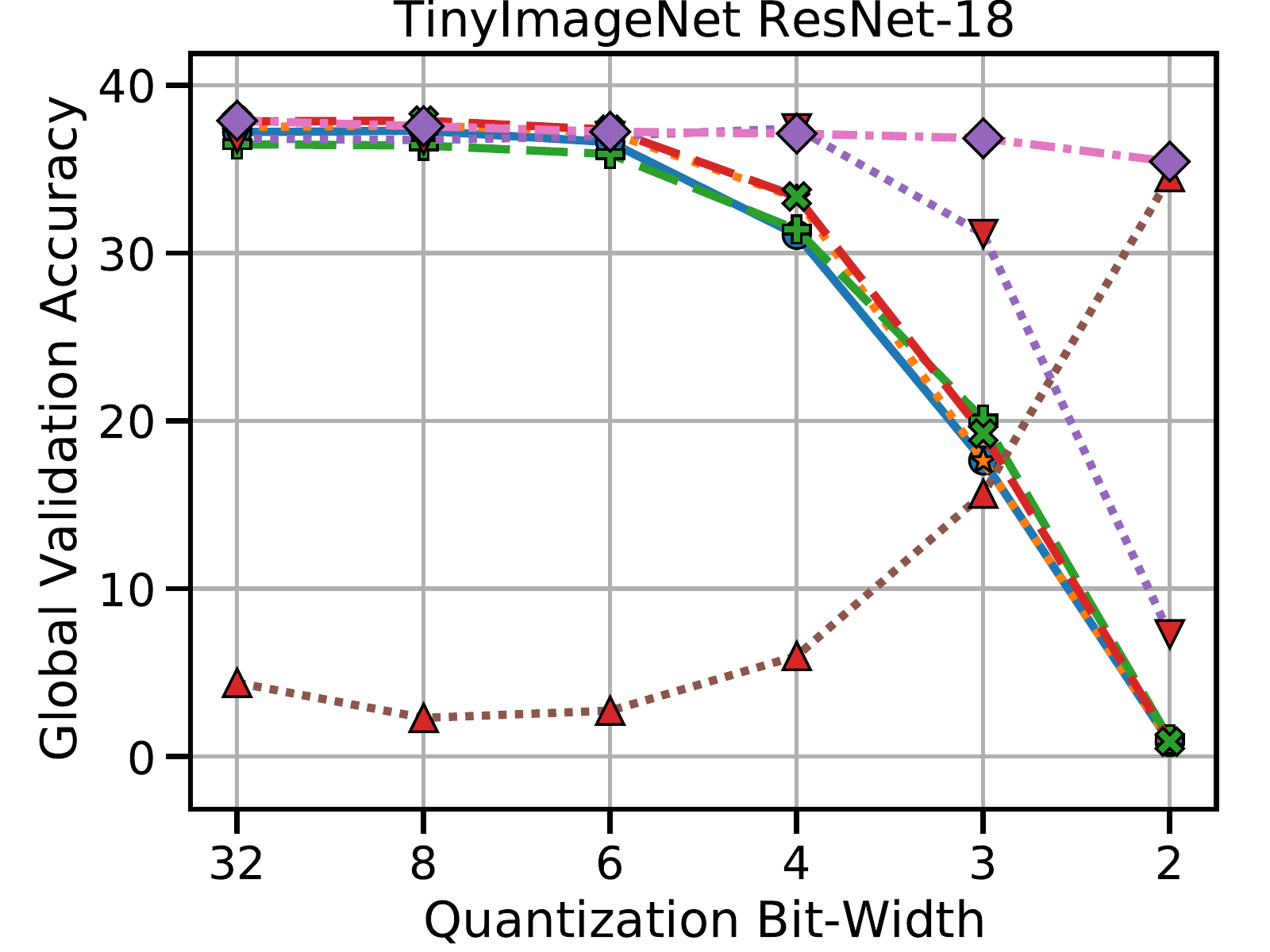}
         \caption{}
         \label{fig:tinyimagenetresnet18_wqr}
     \end{subfigure}
    \vspace{-1ex}
        \caption{\em Global validation accuracy of the proposed \fedavg variants at different bit-widths of weight quantization for models trained on (a) \cifar{-10} using \sresnet{-20}, (b) \cifar{-10} using \slenet{-5}, (c) \cifar{-100} using \sresnet{-20}, and (d) \tinyimagenet{} using the \sresnet{-18} architecture. Here, an abbreviation of ``W'' indicates that we performed weight-quantization only, whereas ``W-$\#$'' refers to quantization at $\#$ bits.}
        \label{fig:wqr_comparisons}
\end{figure*}

\subsection{Weight Quantization}

Firstly, we compare all weight-only quantization-robust \fedavg variants introduced in this work, \emph{i.e.}, we keep the activations in full precision. We report validation accuracy at the different levels of weight quantization for various datasets and network combinations in \figref{fig:wqr_comparisons}.


\paragraph{\cifar{-10}.}
\figref{fig:cifar10resnet20_wqr} shows the performance of \fedavg variants on \cifar{-10} with \sresnet{-20}.
We see that the classification acurracy for the baseline \fedavg trained model drops considerably when weights are quantized to low bit-widths ($2$ and $3$). Both, \fedavg-\acrshort{KURE} and \fedavg-\acrshort{APQN} perform similarly to the baseline with only slight improvements at low bit-widths. Our \fedavg-\acrshort{QAT} variants outperforms the other \fedavg variants. A signficant drop in validation accuracy is observed as the \fedavg-\acrshort{QAT} models are quantized to bit-widths other than the target bit-width they have been trained for. This is a known issue with \acrshort{QAT} trained models in the context of quantization robustness. Our proposed \fedavg-\acrshort{MQAT} variant directly targets this issue. The \fedavg-\acrshort{MQAT} model outperforms all other variants at different level of quantizations consistently. Furthermore, it improves validation accuracy at full precision as well. Further investigation revealed that for an over-parameterized network such as \sresnet{-20}, the \fedavg baseline overfits the training set. Our proposed \fedavg-\acrshort{MQAT} implicitly regularizes the \acrshort{FL} model and avoids the issue of overfitting in this experimental setup and thus achieves better full precision accuracy. We further investigate overfitting and the implicit regularization phenomenon of \fedavg-\acrshort{MQAT} in Appendix.

\figref{fig:cifar10lenet5_wqr} shows performance on \cifar{-10} with the \slenet{-5} architecture.
Since \slenet{-5} is a relatively small network, no overfitting is observed for the \fedavg baseline. It is important to note that the \slenet{-5} (achieves $48$-$55$ \% accuracy at $2$ bits) is more robust to weight quantization compared to \sresnet{-20} (achieves $28$-$35$\% accuracy at $2$ bits) for the baseline \fedavg trained model on \cifar{-10}. Similar to previous comparisons, \fedavg-\acrshort{KURE} and \fedavg-\acrshort{APQN} achieve marginal gains at low bit-widths. The \fedavg-\acrshort{MQAT} produces better validation accuracy (improvement of $\approx12\%$ at ultra low bit-width of $2$ bits but with considerable loss ($\approx 3\%$) in accuracy at full precision. We believe this is because a small network, such as \slenet{-5} on \cifar{-10}, is inherently hard to compress (quantize) and the gains at low bit-widths come at the cost of considerable degradation in full precision accuracy.

\paragraph{\cifar{-100}.}
\figref{fig:cifar100resnet20_wqr} shows the performance of our \fedavg variants with weight-only quantization on \cifar{-100} using the \sresnet{-20} architecture. Similar to our \cifar{-10} setup, we observe a large drop in the baseline model accuracy at low bits. Our \fedavg-\acrshort{MQAT} variant achieves significant gains, especially at low bit-widths; $\approx 35\%$ at $2$-bits and $\approx 10\%$ at $3$-bits, compared to the baseline \fedavg  model. It is important to note that these improvements on low bit-widths come at the cost of a small loss of accuracy ($\approx 2\%$) at full precision in comparison to the baseline. 


\paragraph{\tinyimagenet{}.} We also performed an experimental evaluation on a more challenging \acrshort{FL} setup; classification on \tinyimagenet{} with the \sresnet{-18} architecture. The results for all of our proposed \fedavg variants can be seen in \figref{fig:tinyimagenetresnet18_wqr}. Despite being the most challenging task with 200 classes and only 500 training samples for each class, our \fedavg-\acrshort{MQAT} variant remains robust to bit-widths as low as $2$-bits without any loss in full-precision accuracy. In comparison to the baseline, it improves the $2$-bit, $3$-bit quantization accuracy by a significant margin ($20-35$\%). We would like to point out that despite being trained for different set of bit-width, our \fedavg-\acrshort{MQAT} variant can achieve accuracy of \fedavg-\acrshort{QAT} for their target bit-widths while preserving the full precision model accuracy.

\paragraph{\femnist{}.}
The \femnist{} dataset has become one of the standard datasets used to evaluate \acrshort{FL} algorithms. For the sake of completeness, we performed quantization robustness experiments on \femnist{} using the \slenet{-5} architecture. We observe that the baseline \fedavg trained model is already robust to various quantization levels, even up to $2$-bits. Compared to our other tasks, we believe classification on \femnist{} with \slenet{-5}, is relatively easier and more robust to quantization.

\begin{figure*}[ht]
     \centering
     \begin{subfigure}[b]{0.45\textwidth}
         \centering
         \includegraphics[width=0.99\textwidth]{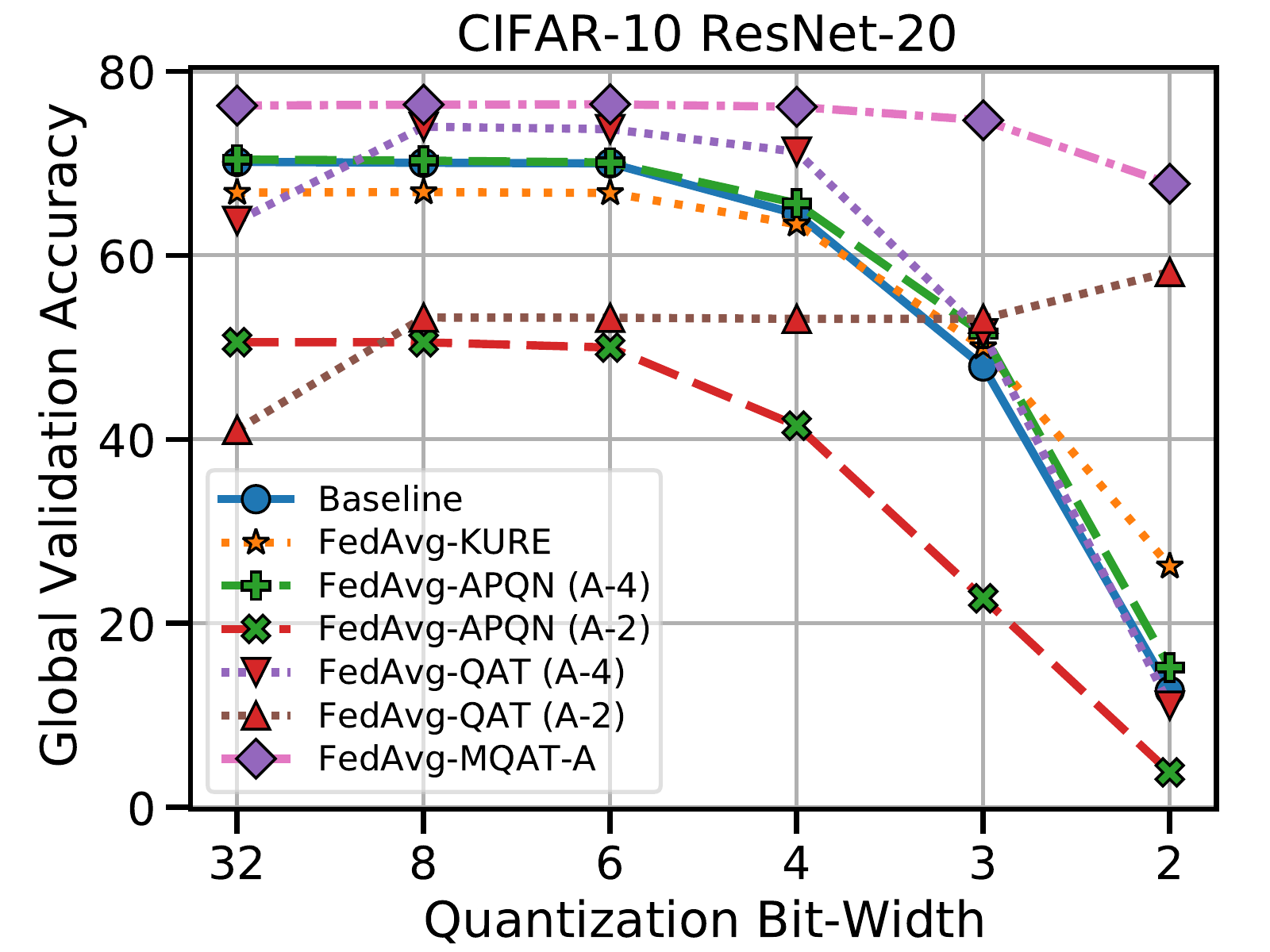}
         \caption{}
         \label{fig:cifar10resnet20_aqr}
     \end{subfigure}
     \begin{subfigure}[b]{0.45\textwidth}
         \centering
         \includegraphics[width=0.99\textwidth]{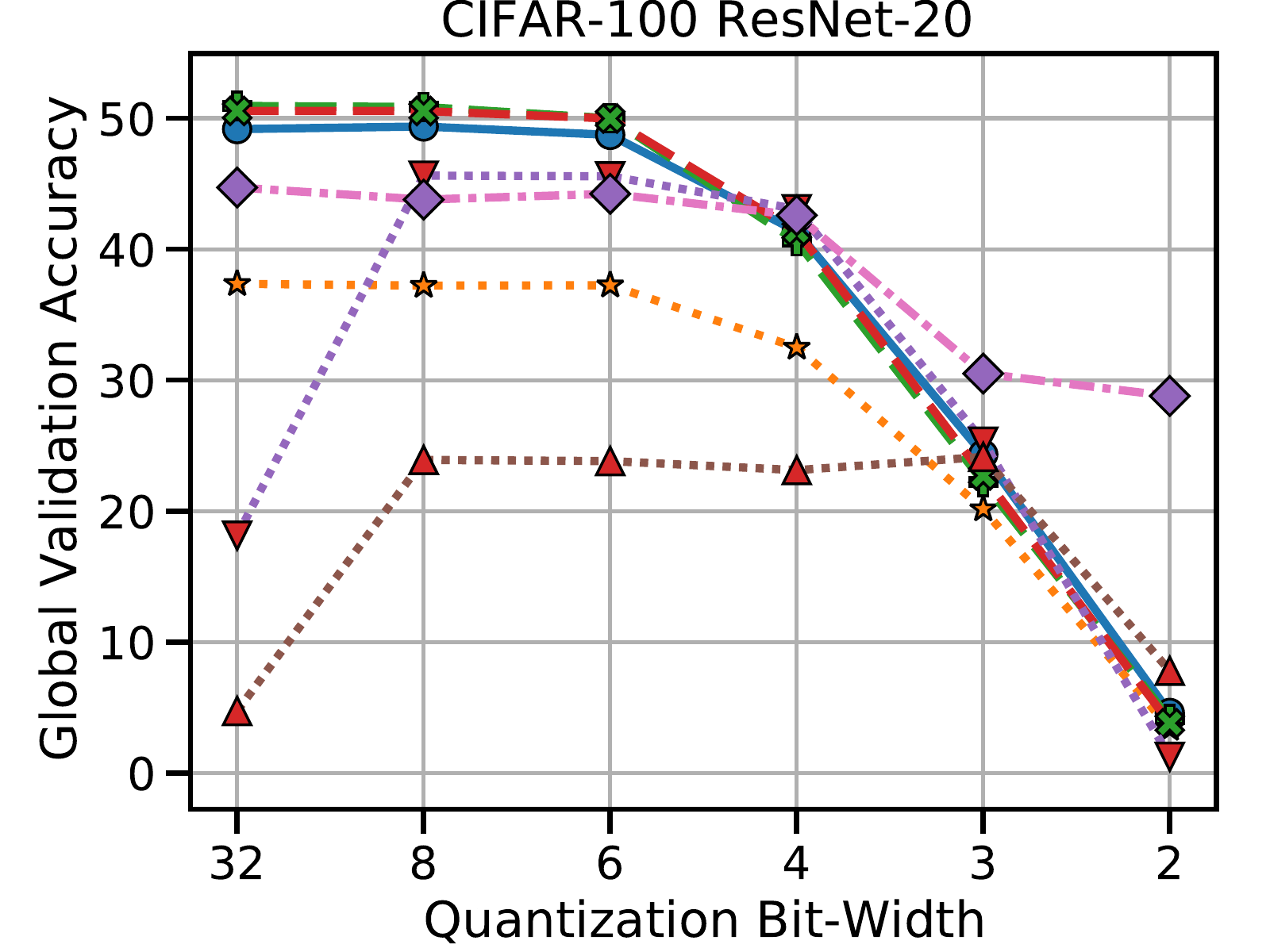}
         \caption{}
         \label{fig:cifar100resnet20_aqr}
     \end{subfigure}
    \vspace{-1ex}

        \caption{\em Global validation accuracy of the proposed \fedavg variants at different bit-widths of activation quantization for models trained on (a) \cifar{-10} using \sresnet{-20}, and (b) \cifar{-100} using \sresnet{-20} architecture. Here, an abbreviation of ``A'' indicates that we performed activation-quantization only, whereas ``A-$\#$'' refers to quantization at $\#$ bits.}
        \label{fig:aqr_comparisons}
\end{figure*}

\subsection{Activation Quantization}

To further demonstrate the effectiveness of our \fedavg variants, we analyze the task of quantizing activations on \cifar{-10} and \cifar{-100} with the \sresnet{-20} architecture. For \fedavg-\acrshort{KURE}, we impose the regularization term on the activations and, in a similar manner, for \fedavg-\{\acrshort{APQN}, \acrshort{QAT} and \acrshort{MQAT}\}, the noise / quantizer is on the activations.

\figref{fig:cifar10resnet20_aqr} and \figref{fig:cifar100resnet20_aqr} show the validation accuracy at different bit-widths for \cifar{-10} and \cifar{-100} respectively.
For both \cifar{-10} and \cifar{-100}, we observe considerable decline in the full precision model accuracy after Kurtosis regularization on activations compared to the baseline. It sould be noted that the original work on \acrshort{KURE}~\cite{kure} considered weight-quantization only. \fedavg-\acrshort{MQAT} achieves significant gains at $2$-bit ($\approx 55\%$ on \cifar{-10} and $\approx 24\%$ on \cifar{-100}) and $3$-bit ($\approx 26\%$ on \cifar{-10} and $\approx 6\%$ on \cifar{-100}) quantization for both \cifar{-10} and \cifar{-100}. 
\begin{wrapfigure}{r}{6cm}
    \includegraphics[width=0.43\textwidth]{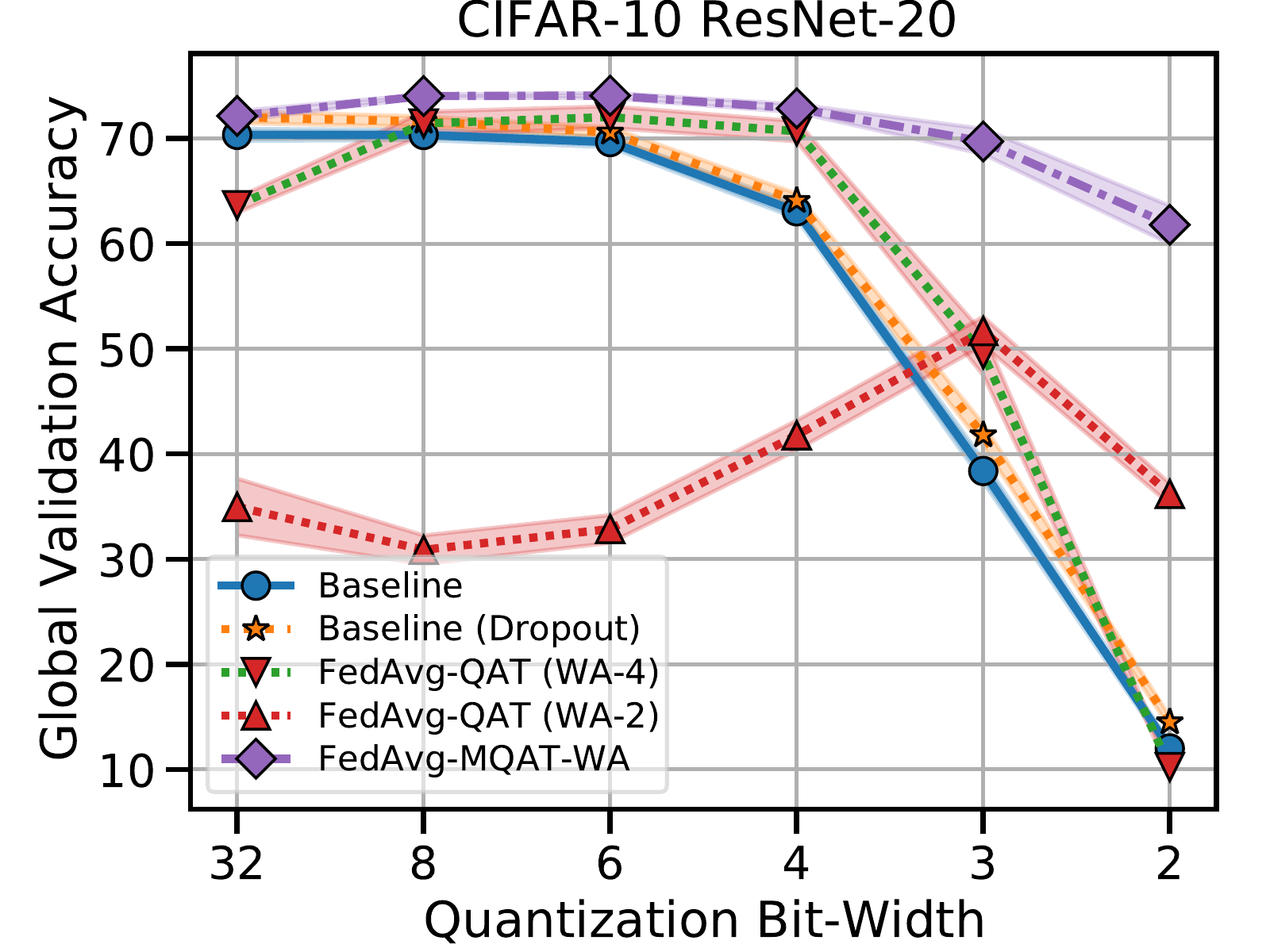}
    \caption{\em Global validation accuracy of the proposed \fedavg variants at different bit-widths of weight and activation quantization for models trained on \cifar{-10} using \sresnet{-20} architecture. An abbreviation of ``WA'' indicates that we performed both weight and activation quantization.
    }
  \vspace{-4ex}
\label{fig:cifar10resnet20_waqr}
\end{wrapfigure}
Perhaps surprisingly, we can see that our \fedavg-\acrshort{MQAT} trained model outperforms the respective \fedavg-\acrshort{QAT} models trained on their own respective bitwidth for $2$ and $4$-bits. For \cifar{-100} dataset, we observe a considerable decrease in model accuracy for \fedavg-\acrshort{MQAT} variant at higher bit-widths (and full precision).

\subsection{Weight and Activation Quantization}

The combination of activation and weight quantization promises to fully harness the advantages of specialized hardware accelerators. 
\figref{fig:cifar10resnet20_waqr} illustrates the impact of our quantization robustness variants for joint weight and activation quantization on \cifar{-10} with the \sresnet{-20} architecture. In this setting, both, weights and activations are quantized before each matrix multiplication.
We see that \fedavg-\acrshort{MQAT} outperforms the \fedavg baseline as well as the bit-specific \fedavg-\acrshort{QAT} across all bit-widths, including full-precision.
As noted before, \sresnet{-20} exhibits overfitting on \cifar{-10}. Thus, we include another baseline which introduces dropout ($50 \%$) before the final fully-connected layer, which marginally improves performance.

\section{Discussion}

Real-world deployments of FL, necessarily require catering to a heterogeneous device landscape. Quantization-robust server models, \emph{i.e.}, models that have robust performance on arbitrary target bit-widths, are a step towards effectively navigating such a landscape. In this work, we introduced several variants of the \fedavg algorithm that encourage quantization-robustness for the server model. Experimentally, we demonstrated that our \fedavg variants can achieve good performance on several target bit-widths, without significant accuracy degradation for the full precision model. No method clearly outperforms in all settings, although we see \acrshort{MQAT} performing well in most situations, especially for lower bit-widths. Theoretically, we showed that quantization-aware local training on the clients provides a convergence rate that is similar to traditional FL, albeit with an extra error floor that depends on the parameters of the quantization procedure and the model characteristics.

In the future, we would focus on end-to-end quantized training which can enable efficiency for both client-server communication as well as on-device training. Further axes of heterogeneity than bit-width should be considered, such as non-uniform quantization, different quantization strategies (symmetric vs. asymmetric) and more subtle differences in inference engines across devices.  We believe that the scenario where a client's hardware characteristics are considered fixed throughout training opens up the possibility for advanced client sub-sampling strategies and necessitates a discussion on the impact of client-specific model performance as a trade-off between efficiency and representation. Device capabilities correlate with the heterogeneous socioeconomic landscape of participating devices and each client's dataset's influence on the learned function will be different.

\bibliography{fl_quant_neurips}

\begin{thebibliography}{10}

\bibitem{alizadeh2020gradient}
Milad Alizadeh, Arash Behboodi, Mart van Baalen, Christos Louizos, Tijmen
  Blankevoort, and Max Welling.
\newblock Gradient l1 regularization for quantization robustness.
\newblock {\em ICLR}, 2020.

\bibitem{amiri2020federated}
Mohammad~Mohammadi Amiri, Deniz Gunduz, Sanjeev~R Kulkarni, and H~Vincent Poor.
\newblock Federated learning with quantized global model updates.
\newblock {\em arXiv preprint arXiv:2006.10672}, 2020.

\bibitem{bengio2013estimating}
Yoshua Bengio, Nicholas L{\'e}onard, and Aaron Courville.
\newblock Estimating or propagating gradients through stochastic neurons for
  conditional computation.
\newblock {\em arXiv preprint arXiv:1308.3432}, 2013.

\bibitem{lsq+}
Yash Bhalgat, Jinwon Lee, Markus Nagel, Tijmen Blankevoort, and Nojun Kwak.
\newblock Lsq+: Improving low-bit quantization through learnable offsets and
  better initialization.
\newblock In {\em Proceedings of the IEEE/CVF Conference on Computer Vision and
  Pattern Recognition (CVPR) Workshops}, June 2020.

\bibitem{caldas2018leaf}
Sebastian Caldas, Sai Meher~Karthik Duddu, Peter Wu, Tian Li, Jakub
  Kone{\v{c}}n{\`y}, H~Brendan McMahan, Virginia Smith, and Ameet Talwalkar.
\newblock Leaf: A benchmark for federated settings.
\newblock {\em arXiv preprint arXiv:1812.01097}, 2018.

\bibitem{caldas2018expanding}
Sebastian Caldas, Jakub Kone{\v{c}}ny, H~Brendan McMahan, and Ameet Talwalkar.
\newblock Expanding the reach of federated learning by reducing client resource
  requirements.
\newblock {\em arXiv preprint arXiv:1812.07210}, 2018.

\bibitem{cohen2019certified}
Jeremy Cohen, Elan Rosenfeld, and Zico Kolter.
\newblock Certified adversarial robustness via randomized smoothing.
\newblock In {\em International Conference on Machine Learning}, pages
  1310--1320. PMLR, 2019.

\bibitem{defossez2021differentiable}
Alexandre D{\'e}fossez, Yossi Adi, and Gabriel Synnaeve.
\newblock Differentiable model compression via pseudo quantization noise.
\newblock {\em arXiv preprint arXiv:2104.09987}, 2021.

\bibitem{esser2021lsq}
Steven~K. Esser, Jeffrey~L. McKinstry, Deepika Bablani, Rathinakumar Appuswamy,
  and Dharmendra~S. Modha.
\newblock Learned step size quantization.
\newblock {\em International Conference on Learning Representations (ICLR)},
  2020.

\bibitem{Gong2019DSQ}
Ruihao Gong, Xianglong Liu, Shenghu Jiang, Tianxiang Li, Peng Hu, Jiazhen Lin,
  Fengwei Yu, and Junjie Yan.
\newblock Differentiable soft quantization: Bridging full-precision and low-bit
  neural networks.
\newblock {\em 2019 IEEE/CVF International Conference on Computer Vision
  (ICCV)}, pages 4851--4860, 2019.

\bibitem{Han2021binreg}
Tiantian Han, Dong Li, Ji~Liu, Lu~Tian, and Yi~Shan.
\newblock Improving low-precision network quantization via bin regularization.
\newblock In {\em Proceedings of the IEEE/CVF International Conference on
  Computer Vision (ICCV)}, pages 5261--5270, October 2021.

\bibitem{he2020group}
Chaoyang He, Murali Annavaram, and Salman Avestimehr.
\newblock Group knowledge transfer: Federated learning of large cnns at the
  edge.
\newblock {\em arXiv preprint arXiv:2007.14513}, 2020.

\bibitem{horvath2021fjord}
Samuel Horvath, Stefanos Laskaridis, Mario Almeida, Ilias Leontiadis,
  Stylianos~I Venieris, and Nicholas~D Lane.
\newblock Fjord: Fair and accurate federated learning under heterogeneous
  targets with ordered dropout.
\newblock {\em arXiv preprint arXiv:2102.13451}, 2021.

\bibitem{hsu2019measuring}
Tzu-Ming~Harry Hsu, Hang Qi, and Matthew Brown.
\newblock Measuring the effects of non-identical data distribution for
  federated visual classification.
\newblock {\em arXiv preprint arXiv:1909.06335}, 2019.

\bibitem{Lee2021EWGS}
B.~Ham J.~Lee, D.~Kim.
\newblock Network quantization with element-wise gradient scaling.
\newblock In {\em Proceedings of the IEEE/CVF Conference on Computer Vision and
  Pattern Recognition}, 2021.

\bibitem{jiang2019model}
Yuang Jiang, Shiqiang Wang, Victor Valls, Bong~Jun Ko, Wei-Han Lee, Kin~K
  Leung, and Leandros Tassiulas.
\newblock Model pruning enables efficient federated learning on edge devices.
\newblock {\em arXiv preprint arXiv:1909.12326}, 2019.

\bibitem{Jangho2020PBGS}
Jangho Kim, KiYoon Yoo, and Nojun Kwak.
\newblock Position-based scaled gradient for model quantization and pruning.
\newblock In H.~Larochelle, M.~Ranzato, R.~Hadsell, M.~F. Balcan, and H.~Lin,
  editors, {\em Advances in Neural Information Processing Systems}, volume~33,
  pages 20415--20426. Curran Associates, Inc., 2020.

\bibitem{krishnamoorthi}
Raghuraman {Krishnamoorthi}.
\newblock {Quantizing deep convolutional networks for efficient inference: A
  whitepaper}.
\newblock {\em arXiv preprint arXiv:1806.08342}, 2018.

\bibitem{krizhevsky2009learning}
Alex Krizhevsky et~al.
\newblock Learning multiple layers of features from tiny images.
\newblock 2009.

\bibitem{li2017training}
Hao Li, Soham De, Zheng Xu, Christoph Studer, Hanan Samet, and Tom Goldstein.
\newblock Training quantized nets: A deeper understanding.
\newblock In {\em Proceedings of the 31st International Conference on Neural
  Information Processing Systems}, pages 5813--5823, 2017.

\bibitem{louizos2021expectation}
Christos Louizos, Matthias Reisser, Joseph Soriaga, and Max Welling.
\newblock An expectation-maximization perspective on federated learning.
\newblock {\em arXiv preprint arXiv:2111.10192}, 2021.

\bibitem{mcmahan2017communication}
Brendan McMahan, Eider Moore, Daniel Ramage, Seth Hampson, and Blaise~Aguera
  y~Arcas.
\newblock Communication-efficient learning of deep networks from decentralized
  data.
\newblock In {\em Artificial intelligence and statistics}, pages 1273--1282.
  PMLR, 2017.

\bibitem{nagel2021white}
Markus Nagel, Marios Fournarakis, Rana~Ali Amjad, Yelysei Bondarenko, Mart van
  Baalen, and Tijmen Blankevoort.
\newblock A white paper on neural network quantization.
\newblock {\em arXiv preprint arXiv:2106.08295}, 2021.

\bibitem{nahshan2020LAPQ}
Yury Nahshan, Brian Chmiel, Chaim Baskin, Evgenii Zheltonozhskii, Ron Banner,
  Alex~M. Bronstein, and Avi Mendelson.
\newblock Loss aware post-training quantization, 2020.

\bibitem{Park2020profit}
Eunhyeok Park and Sungjoo Yoo.
\newblock {PROFIT:} {A} novel training method for sub-4-bit mobilenet models.
\newblock In Andrea Vedaldi, Horst Bischof, Thomas Brox, and Jan{-}Michael
  Frahm, editors, {\em Computer Vision - {ECCV} 2020 - 16th European
  Conference, Glasgow, UK, August 23-28, 2020, Proceedings, Part {VI}}, volume
  12351 of {\em Lecture Notes in Computer Science}, pages 430--446. Springer,
  2020.

\bibitem{reddi2020adaptive}
Sashank Reddi, Zachary Charles, Manzil Zaheer, Zachary Garrett, Keith Rush,
  Jakub Kone{\v{c}}n{\`y}, Sanjiv Kumar, and H~Brendan McMahan.
\newblock Adaptive federated optimization.
\newblock {\em arXiv preprint arXiv:2003.00295}, 2020.

\bibitem{reisizadeh2020fedpaq}
Amirhossein Reisizadeh, Aryan Mokhtari, Hamed Hassani, Ali Jadbabaie, and
  Ramtin Pedarsani.
\newblock Fedpaq: A communication-efficient federated learning method with
  periodic averaging and quantization.
\newblock In {\em International Conference on Artificial Intelligence and
  Statistics}, pages 2021--2031. PMLR, 2020.

\bibitem{kure}
Moran Shkolnik, Brian Chmiel, Ron Banner, Gil Shomron, Yury Nahshan, Alex
  Bronstein, and Uri Weiser.
\newblock Robust quantization: One model to rule them all.
\newblock {\em NeurIPS}, 2020.

\bibitem{triastcyn2021dp}
Aleksei Triastcyn, Matthias Reisser, and Christos Louizos.
\newblock Dp-rec: Private \& communication-efficient federated learning.
\newblock {\em arXiv preprint arXiv:2111.05454}, 2021.

\bibitem{dorefa}
Shuchang Zhou, Zekun Ni, Xinyu Zhou, He~Wen, Yuxin Wu, and Yuheng Zou.
\newblock Dorefa-net: Training low bitwidth convolutional neural networks with
  low bitwidth gradients.
\newblock {\em arXiv preprint arXiv:1606.06160}, 2016.

\end{thebibliography}
\bibliographystyle{plain}

\appendix
\onecolumn
\section*{Appendix}
\label{sec:Appendix}
\setcounter{assumption}{0}

\section{Convergence analysis for quantization aware local training in federated learning}
\label{sec:AppendixProof}
In this appendix we provide a convergence analysis for quantization-aware training in combination with federated learning. We follow the analysis presented in \cite{reddi2020adaptive}, while handling the quantization noise, such as~\cite{li2017training}. This appendix tries to be verbose so as to be easy to follow along. Theorem \myref{3.1} contains the formal claim for this proof.
On a high level, we will aim to upper-bound the gradient magnitude of the FL objective (Eq.~\myref{1} in the paper), $\nabla_\bfw F(\bfw_t)$ for different perspectives on $t$. We begin the proof by making explicit assumptions about the objective as well as the nature of quantization. In terms of techniques we rely only on some general inequalities that we detail before the actual proof. We advise the studious reader to keep a second copy of the pdf open to refer back to these inequalities. Furthermore we preface this section by re-iterating notation.
\subsection{Notation and definitions}
\begin{align*}
    \bfw \in \R^D &\quad \dots\text{is the flattened vector of model parameters.}\\
    \bfw^t_{sk} \in \R^D &\quad \dots\text{model parameters at the $k$'th (out of $K$) iteration of local updates}\\
    &\quad \quad \text{on client $s$ in communication round $t$.}\\
    \mathcal{S} & \quad \dots \text{is the set of all clients ('shards'), of size } S = |\mathcal{S}| \\
    F(\bfw) &=  \frac{1}{S}\sum_s^S F_s(\bfw), \quad \text{We use }  \frac{1}{S}\sum_s F_s(\bfw) \text{ as shorthand.}\\
    F_s(\bfw) &= \frac{1}{D_s}\sum_i^{D_s} f_s(\bfw,\epsilon_i), \quad \text{where } D_s = |\mathcal{D}_s| \text{ is the size of the local dataset }\mathcal{D}_s \text{ at client }s.\\
    f_s(\bfw,\epsilon_i) &\quad \dots \text{is a loss function evaluated on a random mini-batch }\epsilon_i \text{ on client }s \text{ with }\\
    &\quad \quad \text{model parameters }\bfw.\\
    \mathbb{E}_t[F(\bfw_{t+1})] &\quad \dots \text{is the loss evaluated on parameters }\bfw \text{ at round } t+1 \text{, in expectation over the}\\
    &\quad \phantom{\dots} \text{randomness at round $t$ that influences the transition of parameters from round } t \\
    &\quad \phantom{\dots} \text{to round } t+1. \\
    \mathbb{E}[F(\bfw)] &\quad \dots \text{is the loss evaluated on parameters } \bfw, \text{ averaged over all possible sources of} \\ &\quad \phantom{\dots} \text{randomness. Precise definition given in context below.}\\
    \nabla_\bfw F(\bfw) &\quad \dots \text{is the gradient of }F\text{ with respect to }\bfw\text{ evaluated at }\bfw. \\ &\quad \phantom{\dots}\text{We use the shorthand }\nabla F(\bfw)\text{, similarly for }\nabla F_s(\bfw) \text{ and } \nabla f_s(\bfw).\\
    \Delta_b & \quad \dots \text{ is the step size of the quantizer / magnitude of the quantization noise for a}\\
    &\quad \phantom{\dots} \text{given bit-width $b$.}\\
\end{align*}

\subsection{Assumptions}
\begin{assumption}[Lipschitz Gradient]\label{as:lipschitz}
Each local loss function $F_s$ is $L$-smooth $\forall s \in \mathcal{S}$, \emph{i.e.}, $\|\nabla F_s(\vx) - \nabla F_s(\vy)\| \leq L\|\vx - \vy\|$, $\forall \vx, \vy \in \R^D$.
\end{assumption}
\begin{assumption}[Bounded variance] \label{as:boundedVariance}
Each $F_s$ has bounded local variance, \emph{i.e.}, $\E[\|\nabla f_s(\bfw, \epsilon) - \nabla F_s(\bfw)\|^2] \leq \sigma^2_l$, where $f_s$ is a stochastic estimate of the loss based on an $\bfw \in \R^D$ and $\epsilon$ is a random mini-batch. Furthermore, the global variance is also bounded, \emph{i.e.}, $\frac{1}{S}\sum_s \|\nabla F_s(\bfw) - \nabla F(\bfw)\|^2 \leq \sigma^2_g$, $\forall \bfw \in \R^D$.
\end{assumption}
\begin{assumption}[Bounded quantization noise]\label{as:boundedQuantNoise}
Let $\bfw$ be a variable to be quantized, $j$ be any of its dimensions and $Q(\cdot)$ the quantization operation. We assume that the noise added in order to quantize $w_j$, \emph{i.e.}, $r_j = Q(w_j) - w_j$ is bounded by half the step size of the quantizer, i.e., $\frac{\Delta}{2}$.
\end{assumption}

\subsection{Auxiliary lemmata / inequalities}
In this section we will provide some inequalities and lemmata that will be useful for the proof of our main theorem. 
\begin{lemma}\label{lemma:L1}
For any $\gamma > 0$ we have that $\pm2\alpha\beta \leq \gamma \alpha^2 + \frac{1}{\gamma}\beta^2$
\end{lemma}
\begin{proof}
\begin{align}
    0 \leq (\sqrt{\gamma}\alpha \pm \frac{1}{\sqrt{\gamma}}\beta)^2 \rightarrow
    0 \leq \gamma \alpha^2 + \frac{1}{\gamma}\beta^2 \pm 2\alpha\beta \rightarrow
    \pm2\alpha\beta \leq \gamma \alpha^2 + \frac{1}{\gamma}\beta^2. \nonumber
\end{align}
\end{proof}

\begin{corollary}\label{corollary:C1}
For any $\gamma > 0$ we have that $(\alpha \pm \beta)^2 \leq (1 + \gamma)\alpha^2 + (1 + \frac{1}{\gamma})\beta^2$
\end{corollary}
\begin{proof}
\begin{align}
(\alpha \pm \beta)^2 & = \alpha^2 + \beta^2 \pm    2\alpha\beta\\
\shortintertext{From Lemma \ref{lemma:L1}}
&\leq \alpha^2 + \beta^2 + \gamma \alpha^2 + \frac{1}{\gamma}\beta^2 = (1 + \gamma)\alpha^2 + (1 + \frac{1}{\gamma})\beta^2
\end{align}
\end{proof}

\begin{lemma}\label{lemma:L2}
For random variables $\vz_r, \dots, \vz_r$ we have that 
\begin{align}
    \E[\|\vz_1 + \dots + \vz_r\|^2] \leq r \E[\|\vz_1\|^2 + \dots + \|\vz_r\|^2].
\end{align}
\end{lemma}
\begin{proof}
The proof follows from expanding the square and applying Lemma~\ref{lemma:L1} to each $2z_{i}z_{j}$ term with $\gamma = 1$.
\end{proof}

\begin{lemma}\label{lemma:L3}
Let $\vr$ be the quantization noise added to $\bfw$ satisfying assumption~\eqref{as:boundedQuantNoise}. When performing \acrshort{QAT}, \acrshort{MQAT}, \acrshort{APQN} we have that
\begin{align}
    \E[\|\vr\|^2] \leq DR^2,
\end{align}
where $R = \frac{\Delta_b}{2}$ for \acrshort{QAT}, $R = \max_b \frac{\Delta_b}{2}$ for \acrshort{MQAT} and $R = \frac{\Delta_b}{\sqrt{12}}$ for \acrshort{APQN}.
\end{lemma}
\begin{proof}
For \acrshort{QAT} we have that 
\begin{align}
\E[\|\vr\|^2] = \E[\sum_{d=1}^D r_d^2] \leq \E[\sum_{d=1}^D\frac{\Delta_b^2}{4}] = D\left(\frac{\Delta_b}{2}\right)^2,
\end{align}
due to the bounded quantization noise assumption. For \acrshort{MQAT} where we consider multiple bitwidths, it suffices to pick the largest $\Delta_b$ to satisfy this upper bound. For \acrshort{APQN} we have that 
\begin{align}
    \E[\|\vr\|^2] = \E[\sum_{d=1}^D r_d^2] = \sum_{d=1}^D \E[r_d^2] = \sum_{d=1}^D\frac{\Delta_b^2}{12} = D\left(\frac{\Delta_b}{\sqrt{12}}\right)^2,
\end{align}
due to $r_d \sim \calU\left[-\frac{\Delta_b}{2},\frac{\Delta_b}{2}\right]$ and $\E[r_d^2] = \Var[r_d] + \E[r_d]^2 = \frac{\Delta_b^2}{12}$.
\end{proof}

\begin{lemma}\label{lemma:L4}
For any local learning rate $\eta_c \leq \frac{1}{10LK}$, we can bound the difference between the local shadow weights and the server weights for any $k \in \{0, \dots, K-1\}$ and $K \geq 1$ at a given federated training iteration $t$ via
\begin{align}
    \frac{1}{S}\sum_s \E\|\bfw_{sk}^t - \bfw_t\|^2 \leq 4K\eta^2_c(\sigma^2_l + 6K\sigma^2_g) + 32K^2\eta^2_cL^2DR^2 + 24K^2\eta^2_c\|\nabla F(\bfw_t)\|^2.
\end{align}
\end{lemma}
\begin{proof}
We begin by noting that 
\begin{align}
    \E\|\bfw_{sk}^t - \bfw_t\|^2 & = \E\|\underbrace{\bfw^t_{s,k-1} - \bfw_t}_{a} - \eta_c(\underbrace{\nabla f_s(\bfw^t_{s,k-1} + \vr_{s,k-1}^t) - \nabla F_s(\bfw^t_{s,k-1} + \vr_{s,k-1}^t)}_{b} \nonumber \\ & + \underbrace{\nabla F_s(\bfw^t_{s,k-1} + \vr_{s,k-1}^t) - \nabla F_s(\bfw_t)}_{c} + \underbrace{\nabla F_s(\bfw_t) - \nabla F(\bfw_t)}_d + \underbrace{\nabla F(\bfw_t)}_e)\|^2,
\end{align}
where we introduced several shorthand notations for easier manipulation of the terms. By expanding the norm using the multinomial theorem we have that
\begin{align}
    & = \E[\|a\|^2] + \eta^2_c\E[\|b\|^2] + \eta_c^2\E[\|c\|^2] +\eta_c^2\E[\|d\|^2] + \eta_c^2\E[\|e\|^2] \nonumber\\
    & - 2\E[a^T(\eta_c(b + c + d + e))] + 2\eta^2_c\E[b^T(c + d + e)] \nonumber \\ & +2\eta^2_c\E[c^T(d + e)] + 2\eta_c^2\E[d^T e].
\end{align}
We can now use the fact that the expectation of $b$ is zero, since $\nabla f_s(\bfw^t_{s,k-1} + \vr_{s,k-1}^t)$ is an unbiased estimate of $\nabla F_s(\bfw^t_{s,k-1} + \vr_{s,k-1}^t)$. In this way, we can simplify the above to
\begin{align}
    & = \E[\|a\|^2] + \eta^2_c\E[\|b\|^2] + \eta_c^2\E[\|c\|^2] +\eta_c^2\E[\|d\|^2] + \eta_c^2\E[\|e\|^2] \nonumber\\
    & - 2\E[a^T(\eta_c(c + d + e))] + 2\eta_c^2\E[c^T(d + e)] + 2\eta_c^2\E[d^T e].
\end{align}
We can now use Lemma~\ref{lemma:L1} with a $\gamma = 2K - 1$ in order to ``split'' the $2\E[a^T(\eta_c(c + d + e))]$ term
\begin{align}
    & \leq \E[\|a\|^2] + \eta^2_c\E[\|b\|^2] + \eta_c^2\E[\|c\|^2] +\eta_c^2\E[\|d\|^2] + \eta_c^2\E[\|e\|^2] \nonumber \\
    & + \frac{1}{2K - 1}\E[\|a\|^2] + (2K-1)\eta_c^2\E[\|c + d + e\|^2] \nonumber \\ & + 2\eta_c^2\E[c^T(d + e)] + 2\eta_c^2\E[d^T e].
\end{align}
Following that, we can see that several terms cancel, due to $\E[\|c + d + e\|^2] = \E[\|c\|^2] + \E[\|d\|^2] + \E[\|e\|^2] + 2\E[c^T(d + e)] + 2\E[d^T e]$
\begin{align}
    & = \left(1 + \frac{1}{2K - 1}\right)\E[\|a\|^2] + \eta^2_c\E[\|b\|^2] + 2K\eta_c^2\E[\|c + d + e\|^2].
\end{align}
Finally, we will apply Lemma~\ref{lemma:L2} in order to split $\E[\|c + d + e\|^2]$ and thus end up with
\begin{align}
    & \leq \left(1 + \frac{1}{2K - 1}\right)\E[\|a\|^2] + \eta_c^2\E[\|b\|^2] + 6K\eta_c^2\E[\|c\|^2] \nonumber \\ &+ 6K\eta_c^2\E[\|d\|^2] + 6K\eta_c^2\E[\|e\|^2],\\
    & = \left(1 + \frac{1}{2K - 1}\right)\E[\|\bfw^t_{s,k-1} - \bfw_t\|^2] \nonumber \\
    &+\eta^2_c\E[\|\nabla f_s(\bfw^t_{s,k-1} + \vr_{s,k-1}^t) - \nabla F_s(\bfw^t_{s,k-1} + \vr_{s,k-1}^t)\|^2]\nonumber \\ 
    & + 6K\eta^2_c\E[\|\nabla F_s(\bfw^t_{s,k-1} + \vr_{s,k-1}^t) - \nabla F_s(\bfw_t)\|^2]\nonumber\\
    & +6K\eta^2_c\E[\|\nabla F_s(\bfw_t) - \nabla F(\bfw_t)\|^2] + 6K\eta_c^2\E[\|\nabla F(\bfw_t)\|^2],
\end{align}
where we replaced the shorthand notations with their original terms. To proceed, we will make use of assumptions \ref{as:boundedVariance} and \ref{as:lipschitz} to arrive at
\begin{align}
    & \leq \left(1 + \frac{1}{2K - 1}\right)\E[\|\bfw^t_{s,k-1} - \bfw_t\|^2] + \eta^2_c\sigma^2_l \nonumber\\ 
    & + 6K\eta^2_c L^2\E[\|\bfw^t_{s,k-1} + \vr_{s,k-1}^t - \bfw_t\|^2]\nonumber\\
    & +6K\eta^2_c\E[\|\nabla F_s(\bfw_t) - \nabla F(\bfw_t)\|^2] + 6K\eta_c^2\E[\|\nabla F(\bfw_t)\|^2].
\end{align}
We can now make use of corollary~\ref{corollary:C1} with a $\gamma = 3$ in order to separate the quantization error $\vr_{s,k-1}$ from the difference between the local (shadow) weight and the server weight
\begin{align}
    & \leq \left(1 + \frac{1}{2K - 1}\right)\E[\|\bfw^t_{s,k-1} - \bfw_t\|^2] + \eta^2_c\sigma^2_l \nonumber\\ 
    & + 24K\eta^2_c L^2\E[\|\bfw^t_{s,k-1} - \bfw_t\|^2] + 8K\eta^2_c L^2\E[\|\vr_{s,k-1}^t \|^2]\nonumber\\
    & +6K\eta^2_c\E[\|\nabla F_s(\bfw_t) - \nabla F(\bfw_t)\|^2] + 6K\eta_c^2\E[\|\nabla F(\bfw_t)\|^2],
    \shortintertext{and then use Lemma~\ref{lemma:L3} in order to bound the squared norm of $\vr_{s,k-1}$}
    & \leq \left(1 + \frac{1}{2K - 1}\right)\E[\|\bfw^t_{s,k-1} - \bfw_t\|^2] + \eta^2_c\sigma^2_l \nonumber\\ 
    & + 24K\eta^2_c L^2\E[\|\bfw^t_{s,k-1} - \bfw_t\|^2] + 8K\eta^2_c L^2DR^2\nonumber\\
    & +6K\eta^2_c\E[\|\nabla F_s(\bfw_t) - \nabla F(\bfw_t)\|^2] + 6K\eta_c^2\E[\|\nabla F(\bfw_t)\|^2].
\end{align}
To proceed and make further use of our assumptions, we will average the aformentioned inequality over the clients and thus have that 
\begin{align}
    \frac{1}{S}\sum_s\E[\|\bfw^t_{sk} - \bfw_t\|^2] &\leq \left(1 + \frac{1}{2K - 1} + 24K\eta^2_cL^2 \right)\frac{1}{S}\sum_s\E[\|\bfw^t_{s,k-1} - \bfw_t\|^2] + \eta^2_c\sigma^2_l \nonumber\\ 
    & + 8K\eta^2_c L^2DQ^2 +6K\eta^2_c\frac{1}{S}\sum_s\E[\|\nabla F_s(\bfw_t) - \nabla F(\bfw_t)\|^2] \\
    & + 6K\eta_c^2\E[\|\nabla F(\bfw_t)\|^2],
\end{align}
and then make use of assumption~\ref{as:boundedVariance} in order to bound the ``global'' variance
\begin{align}
    &\leq \left(1 + \frac{1}{2K - 1} + 24K\eta^2_cL^2 \right)\frac{1}{S}\sum_s\E[\|\bfw^t_{s,k-1} - \bfw_t\|^2] + \eta^2_c\sigma^2_l \nonumber\\ 
    & + 8K\eta^2_c L^2DQ^2 +6K\eta^2_c\sigma^2_g + 6K\eta_c^2\E[\|\nabla F(\bfw_t)\|^2].
\end{align}
Finally, given our assumption that $\eta_c \leq \frac{1}{10LK}$, we have that $(1 + \frac{1}{2K - 1} + 24K\eta^2_cL^2) \leq (1 + \frac{1}{K-1})$ and thus we can simplify the upper bound even further
\begin{align}
    \frac{1}{S}\sum_s\E[\|\bfw^t_{sk} - \bfw_t\|^2]  & \leq \left(1 + \frac{1}{K-1}\right)\frac{1}{S}\sum_s\E[\|\bfw^t_{s,k-1} - \bfw_t\|^2] + \eta^2_c(\sigma^2_l + 6K\sigma^2_g) \nonumber\\ 
    & + 8K\eta^2_c L^2DQ^2 + 6K\eta_c^2\E[\|\nabla F(\bfw_t)\|^2].\label{eq:intermediate_bound_lemm4}
\end{align}
We have now arrived at a point where the average difference between the local shadow weight at iteration $k$ and the server weight is upper bounded by two things; the average difference at iteration $k-1$ along with some constant terms that are independent of the actual weights or iteration. We can thus continue further by sequentially applying the bound at Eq.~\ref{eq:intermediate_bound_lemm4} on each weight difference, up until we end up at the server weight $\bfw_t$ (since local optimization started from that point) where the difference is zero. Notice that each application of this bound ``adds'' additional non-negative terms, so we have that the upper bound of $K$ iterations would upper bound the bound on $K-1$ iterations. Therefore, the ``worst-case'' upper bound is the one where $k = K$. In this case, we can unroll the recursion and have that
\begin{align}
     \frac{1}{S}\sum_s\E[\|\bfw^t_{sk} - \bfw_t\|^2] & \leq \sum_{j = 0}^{K-1}\left(1 + \frac{1}{K-1}\right)^j\big(\eta^2_c(\sigma^2_l + 6K\sigma^2_g) + 8K\eta^2_c L^2DQ^2 \nonumber\\ 
     & +  6K\eta_c^2\E[\|\nabla F(\bfw_t)\|^2]\big).
\end{align}
To simplify even further, we can use the fact that $(1 + \frac{1}{K - 1})^j$ is monotonic in $j$ and that there are $K$ terms in the sum, thus get
\begin{align}
    & \leq K\left(1 + \frac{1}{K - 1}\right)^K\left(\eta^2_c(\sigma^2_l + 6K\sigma^2_g) + 8K\eta^2_c L^2DQ^2 +  6K\eta_c^2\E[\|\nabla F(\bfw_t)\|^2]\right)\\
    \shortintertext{and since $(1 + \frac{1}{K-1})^K \leq 4$ for any $K > 1$}
    &\leq 4K\eta^2_c(\sigma^2_l + 6K\sigma^2_g) + 32K^2\eta^2_c L^2DQ^2 + 24K^2\eta_c^2\E[\|\nabla F(\bfw_t)\|^2],
\end{align}
which proves our claim.
\end{proof}

\subsection{Proof of Theorem~\myref{3.1}}
We begin by noting that the server-side update rule of the model in the case when the clients perform Quantization Aware (QA) SGD with a learning rate of $\eta_c$ and the server does SGD with a learning rate $\eta_s$. The extension of the proof to more involved server-side update rules as in \cite{reddi2020adaptive} is straightforward.
\begin{align}
    \bfw_{t+1} - \bfw_t = \eta_s \vG_t, \qquad \vG_t = -\frac{\eta_c}{S} \sum_s \sum_k \nabla f_s(\bfw^t_{sk} + \vr^t_{sk}),
\end{align}
where $s$ indexes the clients and $S$ is the total number of clients. $k$ indexes the local client iteration number, and we assume that there are $K$ local iterations in total. $\bfw^t_{sk} \in \R^D$ corresponds to a real valued local shadow weight at iteration $t$ and $\vr^t_{sk} \in \R^D$ corresponds to the quantization noise that is added to it in each iteration of the local optimization (as the weights are rounded / noised before the forward pass).

Using the $L$-smoothness of the global loss function,
\begin{align}
    F(\bfw_{t+1}) &\leq F(\bfw_t) + \eta_s \nabla F(\bfw_t)^T\vG_t + \frac{L}{2}\|\eta_s \vG_t\|^2\\
    &= F(\bfw_t) + \eta_s \nabla F(\bfw_t)^T\vG_t + \frac{L\eta_s^2}{2}\| \vG_t\|^2.
\end{align}

We then take an expectation over all randomness at time step $t$,
\begin{align}
    \E_t[F(\bfw_{t+1})] \leq F(\bfw_t) + \eta_s\underbrace{\nabla F(\bfw_t)^T\E_t[\vG_t]}_{T_{1t}} + \underbrace{\frac{L\eta_s^2}{2}\E_t[\| \vG_t\|^2]}_{T_{2t}},\label{eq:incomplete_bound_t12}
\end{align}
and work towards upper bounding the $T_{1t}, T_{2t}$ terms separately.

\paragraph{Bounding $T_{1t}$}
\begin{align}
    \nabla F(\bfw_t)^T\E_t[\vG_t] &= \nabla F(\bfw_t)^T\E_t[\vG_t - \eta_c K\nabla F(\bfw_t) + \eta_c K \nabla F(\bfw_t)]\\
    & = -\eta_c K \|\nabla F(\bfw_t)\|^2 + \underbrace{\nabla F(\bfw_t)^T\E_t[\vG_t + \eta_c K \nabla F(\bfw_t)]}_{T_{3t}}.
\end{align}
We will now work towards upper bounding $T_{3t}$
\begin{align}
    T_{3t} & = \nabla F(\bfw_t)^T\E_t[-\frac{\eta_c}{S}\sum_s \sum_k \nabla f_s(\bfw^t_{sk} + \vr^t_{sk}) + \eta_c K \nabla F(\bfw_t)] \\
    & = \nabla F(\bfw_t)^T\E_t[-\frac{\eta_c}{S}\sum_s \sum_k \nabla F_s(\bfw^t_{sk} + \vr^t_{sk}) + \eta_c K \nabla F(\bfw_t)]\\
    & = \eta_c \nabla F(\bfw_t)^T\E_t[-\frac{1}{S}\sum_s \sum_k \nabla F_s(\bfw^t_{sk} + \vr^t_{sk}) + \frac{1}{S}\sum_s \sum_k \nabla F_s(\bfw_t)].
\end{align}
Now by using Lemma~\ref{lemma:L1} with $\gamma = K$ we have that
\begin{align}
    T_{3t} & \leq \frac{\eta_c K}{2} \|\nabla F(\bfw_t)\|^2 + \frac{\eta_c}{2K}\E_t[\|\frac{1}{S}\sum_s \sum_k \nabla F_s(\bfw^t_{sk} + \vr^t_{sk}) - \frac{1}{S}\sum_s \sum_k \nabla F_s(\bfw_t)\|^2]\\
    & = \frac{\eta_c K}{2} \|\nabla F(\bfw_t)\|^2 + \frac{\eta_c}{2KS^2}\E_t[\|\sum_s \sum_k (\nabla F_s(\bfw^t_{sk} + \vr^t_{sk}) - \nabla F_s(\bfw_t))\|^2].
\end{align}
By then using Lemma~\ref{lemma:L2} to push the squared norm inside the sum
\begin{align}
    &\leq \frac{\eta_c K}{2} \|\nabla F(\bfw_t)\|^2 + \frac{\eta_c SK}{2KS^2}\E_t[\sum_s \sum_k \|\nabla F_s(\bfw^t_{sk} + \vr^t_{sk}) - \nabla F_s(\bfw_t))\|^2] \\
    \shortintertext{and from the Lipschitz gradient assumption}
    & \leq \frac{\eta_c K}{2} \|\nabla F(\bfw_t)\|^2 + \frac{\eta_c }{2S}\E_t[\sum_s \sum_k \|L(\bfw^t_{sk} + \vr^t_{sk} - \bfw_t)\|^2]\\
    & = \frac{\eta_c K}{2} \|\nabla F(\bfw_t)\|^2 + \frac{\eta_c L^2 }{2}\sum_k \frac{1}{S}\sum_s \E_t[\|\bfw^t_{sk} + \vr^t_{sk} - \bfw_t\|^2].
\end{align}
We will now once more use Lemma~\ref{lemma:L2} to separate the shadow weight difference from the quantization noise
\begin{align}
    &\leq \frac{\eta_c K}{2} \|\nabla F(\bfw_t)\|^2 + \eta_c L^2 \sum_k \frac{1}{S}\sum_s \E_t[\|\bfw^t_{sk} - \bfw_t\|^2] + \eta_c L^2 \sum_k \frac{1}{S}\sum_s E_t[\|\vr^t_{sk}\|^2]\\
    \shortintertext{and from Lemma~\ref{lemma:L3}}
    & \leq \frac{\eta_c K}{2} \|\nabla F(\bfw_t)\|^2 + \eta_c L^2 \sum_k \frac{1}{S}\sum_s \E_t[\|\bfw^t_{sk} - \bfw_t\|^2] + \eta_c K L^2 D R^2.
\end{align}
We see that in order to proceed, we need to upper bound the difference between the local shadow weight at any iteration $k$ and the server weight. This is were we will use our Lemma~\ref{lemma:L4} in order to proceed
\begin{align}
    & \leq \frac{\eta_c K}{2} \|\nabla F(\bfw_t)\|^2 + K\eta_c L^2(4K\eta^2_c(\sigma^2_l + 6K\sigma^2_g)) + K\eta_cL^2(32K^2\eta^2_cL^2DQ^2)\nonumber \\
    & + K\eta_c L^2(24K^2\eta^2_c\|\nabla F(\bfw_t)\|^2) + \eta_c K L^2 D R^2\\
    & = \left(\frac{\eta_c K }{2} + 24K^3\eta_c^3L^2\right)\|\nabla F(\bfw_t)\|^2 + 4K^2\eta_c^3L^2(\sigma_l^2 + 6K\sigma^2_g) + (\eta_c K + 32\eta^3_c K^3L^2)L^2DR^2.
\end{align}
Finally, we can make use of our assumption $\eta_c \leq \frac{1}{10LK}$ which leads to  $24\eta_c^3L^2K^3 \leq \frac{1}{4}\eta_c K$ and $32\eta_c^3L^2K^3 \leq \frac{1}{3}\eta_c K $. In this way, we can arrive at our final bound for $T_{3t}$ 
\begin{align}
    T_{3t} & \leq \frac{3\eta_c K}{4}\|\nabla F(\bfw_t)\|^2 + \frac{4}{3}\eta_c K L^2DR^2 + 4K^2\eta_c^3L^2(\sigma_l^2 + 6K\sigma^2_g).
\end{align}
Now we can apply this bound to $T_{1t}$ in order to get
\begin{align}
T_{1t} & \leq -\eta_c K \|\nabla F(\bfw_t)\|^2 + \frac{3\eta_c K}{4}\|\nabla F(\bfw_t)\|^2 + \frac{4}{3}\eta_c K L^2DR^2 + 4K^2\eta_c^3L^2(\sigma_l^2 + 6K\sigma^2_g)\\
& = - \frac{\eta_c K}{4}\|\nabla F(\bfw_t)\|^2 + \frac{4}{3}\eta_c K L^2DR^2 + 4K^2\eta_c^3L^2(\sigma_l^2 + 6K\sigma^2_g)
\end{align}
Which is the final upper bound on $T_{1t}$.

\paragraph{Bounding $T_{2t}$} We begin by noting that 
\begin{align}
    \frac{L\eta_s^2}{2}\E_t[\|\vG_t\|^2] & = \frac{L\eta_s^2}{2}\E_t[\|\vG_t - \eta_c K \nabla F(\bfw_t) + \eta_c K \nabla F(\bfw_t)\|^2],
    \intertext{and then we can split the squared norm via Corollary~\ref{corollary:C1} with $\gamma = 1$} 
    & \leq  L\eta_s^2 (\underbrace{\E_t[\|\vG_t + \eta_c K \nabla F(\bfw_t)\|^2]}_{T_{4t}} + \eta_c^2K^2\|\nabla F(\bfw_t)\|^2).
\end{align}
To continue, we will move towards upper bounding $T_{4t}$ by expanding the terms inside the squared norm
\begin{align}
     T_{4t} & = \E_t[\|-\frac{\eta_c}{S}\sum_s \sum_k \nabla f_s(\bfw^t_{sk} + \vr^t_{sk}) + \frac{\eta_c}{S}\sum_s \sum_k \nabla F_s(\bfw_t)\|^2]\\
    &= \frac{\eta_c^2}{S^2}\E_t[\|\sum_s \sum_k\left(\nabla f_s(\bfw^t_{sk} + \vr^t_{sk}) - \nabla F_s(\bfw^t_{sk} + \vr^t_{sk}) + \nabla F_s(\bfw^t_{sk} + \vr^t_{sk}) - \nabla F_s(\bfw_t)\right)\|^2].
\end{align}
We will then apply Lemma~\ref{lemma:L2} in order to move the squared norm inside the sums
\begin{align}
    & \leq \frac{\eta_c^2 S K }{S^2}\sum_s \sum_k \E_t[\|\nabla f_s(\bfw^t_{sk} + \vr^t_{sk}) - \nabla F_s(\bfw^t_{sk} + \vr^t_{sk}) + \nabla F_s(\bfw^t_{sk} + \vr^t_{sk}) - \nabla F_s(\bfw_t)\|^2]\\
    \intertext{and will apply Corollary~\ref{corollary:C1} with $\gamma = 1$ in order to split the norm}
    &\leq \frac{2\eta_c^2 K}{S}\sum_s \sum_k\big(\E_t[\|\nabla f_s(\bfw^t_{sk} + \vr^t_{sk}) - \nabla F_s(\bfw^t_{sk} + \vr^t_{sk})\|^2] + \E_t[\|\nabla F_s(\bfw^t_{sk} + \vr^t_{sk})\nonumber\\
    &- \nabla F_s(\bfw_t)\|^2]\big).
\end{align}
To proceed, we will make use of our assumptions~\ref{as:lipschitz}, ~\ref{as:boundedVariance} in order to get
\begin{align}
    &\leq 2\eta_c^2 K^2\sigma_l^2 + 2\eta_c^2K\sum_k\frac{1}{S}\sum_s\E_t[\|L(\bfw^t_{sk} + \vr^t_{sk} - \bfw_t)\|^2]
    \intertext{and we will apply Corollary~\ref{corollary:C1} with $\gamma = 1$ one more time in order to split the norm of the weight difference and the quantization error}
    &\leq 2\eta_c^2 K^2\sigma_l^2 + 4\eta_c^2KL^2\sum_k\frac{1}{S}\sum_s\E_t[\|\bfw^t_{sk} - \bfw_t\|^2] + 4\eta_c^2KL^2\sum_k\frac{1}{S}\sum_s \E_t[\|\vr^t_{sk}\|^2]\\
    \intertext{so that we can apply Lemma~\ref{lemma:L3} in order to bound the latter}
    &\leq 2\eta_c^2 K^2\sigma_l^2 + 4\eta_c^2KL^2\sum_k\frac{1}{S}\sum_s\E_t[\|\bfw^t_{sk} - \bfw_t\|^2] + 4\eta_c^2K^2L^2DR^2.
\end{align}
By observing the above, we see that we again end up with the average difference between the shadow weights at each iteration $k$ and the server weight. As a result, we can apply Lemma~\ref{lemma:L4} in order to proceed further
\begin{align}
    &\leq 2\eta_c^2 K^2\sigma_l^2 + 4\eta_c^2K^2L^2DQ^2 \nonumber \\ 
    &+ 4\eta_c^2K^2L^2(4K\eta^2_c(\sigma^2_l + 6K\sigma^2_g) + 32K^2\eta^2_c L^2DR^2 + 24K^2\eta_c^2\E[\|\nabla F(\bfw_t)\|^2])\\
    & = 2\eta_c^2 K^2\sigma_l^2 + 4\eta_c^2K^2L^2DQ^2 + 16K^3\eta_c^4 L^2(\sigma^2_l + 6K\sigma_g^2) + 128\eta_c^4L^4K^4DR^2 \nonumber \\
    & + 96\eta_c^4 K^4L^2\|\nabla F(\bfw_t)\|^2.
\end{align}
In order to simplify the aforementioned inequality we will make a use of our assumption on $\eta_c$, namely that $\eta_c \leq \frac{1}{10LK}$. In this way, we will have that $16K^3\eta_c^4 L^2 \leq \frac{1}{6}K\eta_c^2$ along with $128\eta_c^4K^4L^2 \leq \frac{3}{2}\eta_c^2K^2$. Taking these into account, we have that 
\begin{align}
    &\leq (2 \eta_c^2 K^2 + \frac{1}{6}K\eta_c^2)\sigma_l^2 + K^2\eta_c^2\sigma_g^2 + (4\eta_c^2 K^2 + \frac{3}{2}\eta_c^2K^2)L^2DR^2 + 96\eta_c^4 K^4L^2\|\nabla F(\bfw_t)\|^2\\
    \intertext{and due to $4 + \frac{3}{2} < 6$}
    & \leq (2 \eta_c^2 K^2 + \frac{1}{6}K\eta_c^2)\sigma_l^2 + K^2\eta_c^2\sigma_g^2 + 6\eta_c^2K^2L^2DR^2 + 96\eta_c^4 K^4L^2\|\nabla F(\bfw_t)\|^2,
\end{align}
which constitutes our final upper bound on $T_{4t}$. With this bound at hand, we can move back to bounding $T_{2t}$ and thus get
\begin{align}
    T_{2t} & \leq L\eta_s^2((2 \eta_c^2 K^2 + \frac{1}{6}K\eta_c^2)\sigma_l^2 + K^2\eta_c^2\sigma_g^2 + 6\eta_c^2K^2L^2DR^2 +96\eta_c^4 K^4L^2\|\nabla F(\bfw_t)\|^2 \nonumber\\
    & + \eta_c^2K^2\|\nabla F(\bfw_t)\|^2)\\
    & = L\eta_s^2((2 \eta_c^2 K^2 + \frac{1}{6}K\eta_c^2)\sigma_l^2 + K^2\eta_c^2\sigma_g^2 + 6\eta_c^2K^2L^2DR^2) \nonumber \\ 
    & + L\eta_s^2(96\eta_c^4 K^4L^2 + \eta_c^2K^2) \|\nabla F(\bfw_t)\|^2.
\end{align}

Having bounded $T_{1t}$ and $T_{2t}$, we can now apply these bounds to the inequality at Eq.~\ref{eq:incomplete_bound_t12} and thus get
\begin{align}
\E_t[F(\bfw_{t+1})] & \leq F(\bfw_t) + \eta_s\underbrace{\nabla F(\bfw_t)^T\E_t[\vG_t]}_{T_{1t}} + \underbrace{\frac{L\eta_s^2}{2}\E_t[\|\vG_t\|^2]}_{T_{2t}}\\
   & \leq F(\bfw_t) - \frac{\eta_s \eta_c K}{4}\|\nabla F(\bfw_t)\|^2 + \frac{4}{3}\eta_s\eta_c K L^2DR^2 + 4\eta_s K^2\eta_c^3L^2(\sigma_l^2 + 6K\sigma^2_g)\nonumber \\
    & + L\eta_s^2((2 \eta_c^2 K^2 + \frac{1}{6}K\eta_c^2)\sigma_l^2 + K^2\eta_c^2\sigma_g^2 + 6\eta_c^2K^2L^2DR^2) \nonumber \\ 
    & + L\eta_s^2(96\eta_c^4 K^4L^2 + \eta_c^2K^2) \|\nabla F(\bfw_t)\|^2.
\end{align}
To simplify the aforementioned bound, we can once again make use of our condition $\eta_c \leq \frac{1}{10LK}$ which leads to $96\eta_c^4 K^4L^2 \leq \eta_c^2K^2$. In this way, we get that
\begin{align}
    \E_t[F(\bfw_{t+1})] & \leq F(\bfw_t) - \eta_s(\frac{\eta_c K}{4} - 2L\eta_s\eta_c^2K^2)\|\nabla F(\bfw_t)\|^2 \nonumber \\
    & + (4\eta_s K^2L^2\eta_c^3 + L\eta_s^2(2\eta_c^2 K^2 + \frac{1}{6}K\eta_c^2))\sigma_l^2\nonumber \\
    & + (24\eta_s K^2 L^2\eta_c^3 + L\eta_s^2\eta_c^2 K)K\sigma_g^2 + (\frac{4}{3}\eta_s \eta_c K + 6L\eta_s^2\eta_c^2K^2)L^2DR^2\\
    & = F(\bfw_t) - \eta_s \eta_c \underbrace{(\frac{K}{4} - 2L\eta_s \eta_c K^2)}_{A}\|\nabla F(\bfw_t)\|^2 \nonumber \\
    & + \eta_c^2\underbrace{(4\eta_s K^2L^2\eta_c + L\eta_s^2(2K^2 + \frac{K}{6}))}_{B}\sigma_l^2\nonumber \\
    & + \eta_c^2\underbrace{(24\eta_s K^2 L^2 \eta_c + L\eta_s^2K)}_{\Gamma}K\sigma_g^2 + \eta_c^2\underbrace{(\frac{4\eta_s}{3\eta_c}K + 6L\eta_s^2K^2)}_{H}L^2DR^2,
\end{align}
where we introduced several shorthand notations for easier manipulation of the inequalities. We can thus now re-arrange the inequality to
\begin{align}
    \E_t[F(\bfw_{t+1})] - F(\bfw_t) & \leq -\eta_s\eta_c A\|\nabla F(x_t)^2\| + \eta_c^2(B\sigma^2_l + \Gamma K\sigma_g^2) + \eta_c^2 H L^2DR^2.
\end{align}
In order to consider the entire training trajectory, we will use a telescoping sum, \emph{i.e.}, we will sum this inequality over all rounds and take the expectation at each time-step
\begin{align}
    \sum_{t=1}^T(\E_t[F(\bfw_{t+1})] - \E_{t-1}[F(\bfw_t)]) & \leq -\eta_s\eta_c A \sum_{t=1}^T \|\nabla F(\bfw_t)\|^2 + T\eta_c^2(B\sigma^2_l + \Gamma K \sigma^2_g) \nonumber \\ 
    & + T\eta_c^2 H L^2DR^2.
\end{align}
In doing that, most of the terms on the left-hand-side will cancel across subsequent time-steps and thus we will end up with
\begin{align}
    \E_T[F(\bfw_{T+1})] - F(\bfw_1) & \leq - \eta_s\eta_c A \sum_{t=1}^T\|\nabla F(\bfw_t)\|^2 + T\eta_c^2(B\sigma^2_l + \Gamma K \sigma_g^2) + T\eta_c^2 H L^2DR^2.
\end{align}
We can now re-arrange the terms to get 
\begin{align}
\eta_s\eta_c A \sum_{t=1}^T\|\nabla F(\bfw_t)\|^2 & \leq F(\bfw_1) - \E_T[F(\bfw_{T+1})] + T\eta_c^2(B\sigma^2_l + \Gamma K \sigma_g^2) + T\eta_c^2 H L^2DR^2.
    \intertext{and by considering $\bfw^*$ to be the parameters of lowest loss, $\E_T[F(\bfw_{T+1})] \geq F(\bfw^*)$, we have that}
    \eta_s\eta_c A \sum_{t=1}^T\|\nabla F(\bfw_t)\|^2 & \leq F(\bfw_1) - F(\bfw^*) + T\eta_c^2(B\sigma^2_l + \Gamma K \sigma_g^2) + T\eta_c^2 H L^2DR^2.
\end{align}
In order to proceed, we have to impose a condition on $A$, namely that it has to be positive (otherwise, dividing by $A$ reverses the inequality). For this to happen we need that
\begin{align}
    \frac{K}{4} \geq 2L\eta_s\eta_cK^2 \rightarrow \frac{1}{4} \geq 2L\eta_s\eta_cK \rightarrow \eta_c \leq \frac{1}{8LK\eta_s}.
\end{align}
Assuming that this condition is satisfied, we have that
\begin{align}
    \frac{1}{T}\sum_{t=1}^T\|\nabla F(\bfw_t)\|^2 & \leq \frac{F(\bfw_1) - F(\bfw^*)}{T\eta_s\eta_c A} + \frac{\eta_c}{\eta_s A}(B \sigma^2_l + \Gamma K \sigma^2_g + HL^2DR^2). \\
\end{align}
Finally, in order to obtain the result of Theorem \myref{3.1}, we make use of the fact that $\min_{1 \leq t \leq T}\|\nabla F(\bfw_t)\|^2 \leq \frac{1}{T}\sum_{t=1}^T\|\nabla F(\bfw_t)\|^2$ and thus arrive at
\begin{align}
    \min_{1 \leq t \leq T}\|\nabla F(\bfw_t)\|^2  & \leq \frac{F(\bfw_1) - F(\bfw^*)}{T\eta_s\eta_c A} + \frac{\eta_c}{\eta_s A}(B\sigma^2_l + \Gamma K \sigma^2_g + HL^2DR^2),
\end{align}
which completes the proof.

\section{Experimental Details}
\label{sec:AppendixExpDetails}
We split the data into 100 (\cifar{-10}), 500 (\cifar{-100}), 500 (\tinyimagenet{}), 3500 (\femnist{}) clients in a non-i.i.d way following~\cite{hsu2019measuring}, where in each round only 10 clients participate for all datasets except \tinyimagenet{} dataset, where 100 clients participate. We train different models for 5000 (\cifar{-10} using \sresnet{-20}), 2000 (\cifar{-10} using \slenet{-5}), 10000 (\cifar{-100}), 4500 (\tinyimagenet{}), and 6000 (\femnist{}) rounds. We use small local batch sizes for all clients in our experiments for different datasets: 64 (\cifar{-10}), 20 (\cifar{-100}, \tinyimagenet{}, \femnist{}). For all our experiments, we use ADAM optimizer for server training phase and \acrshort{SGD} optimizer for client training phase. We use single epoch of local client training for each client participating in a round for all our experiments.
For data augmentation, we normalize \cifar{-10} and \cifar{-100} and \tinyimagenet{} to per-channel zero-mean and standard deviation of one. \cifar{-100} further undergoes random cropping to $28$ pixels height and width with zero-padding, followed by random horizontal flipping with $50\%$ probability. \femnist{} requires no preprocessing.

In order to simulate a non-i.i.d. data split that we would expect in the federated scenario, we artificially split \cifar{-10}, \cifar{-100} and \tinyimagenet{} by their label. For \cifar{-10} and \tinyimagenet{}, the label proportions on each client are computed by sampling from a Dirichlet distribution with $\alpha-1.0$ \cite{hsu2019measuring}. For \cifar{-100} we use the coarse labels provided with the dataset and follow \cite{reddi2020adaptive}. For our \femnist{} experiments, the federated split is naturally determined by the writer-id for each client.

We provide the final hyperparameters used for all the experiments in \tabref{tab:hyperparam}. Since the grid-search for \acrshort{FL} is expensive, we tune the hyperparameters such as client learning rate ($\eta_c$), server learning rate ($\eta_s$) and epsilon term ($\eps_s$) used in ADAM optimizer only for the baselines and then use the same set of hyperparameters for all our proposed variants.
For \fedavg-\acrshort{KURE}, we tuned $\lambda$ from the grid $[1e+1, 1e+0, 1e-1, 1e-2, 1e-3, 1e-4]$, found $\lambda=1e-1$ to be optimal and use that for all our experiments. We also provide the details about bit-set used to perform \fedavg-\acrshort{MQAT} variants in \tabref{tab:bitset-mqat}.

\begin{table}[]
\centering
\begin{tabular}{lcccccc}
\toprule
\textbf{Dataset}                                       & \textbf{Network}  & $\eta_s$ & $\eta_c$ & $\eps_s$ & \textbf{Rounds} ($T$) & \textbf{Batch Size} \\
\midrule
\cifar{-10}                                      & \slenet{-5}   & $1e-3$     & $5e-2$     & $1e-8$          & $2000$       & $64$         \\
\cifar{-10}                                      & \sresnet{-20} & $1e-3$     & $5e-2$     & $1e-7$     & 5000       & 64         \\
\cifar{-100}                                     & \sresnet{-20} & $1e-3$     & $5e-2$     & $1e-7$     & 10000      & $20$         \\
\femnist{}                                       & \slenet{-5}   & $1e-3$     & $5e-2$     & $1e-8$      & 6000       & $20$         \\
\tinyimagenet{}                                  & \sresnet{-18} & $1e-2$     & $1e-2$     & $1e-3$     & 4500       & $20$  \\
\bottomrule
\end{tabular}
\vspace{1ex}
\caption{\em Hyperparameters used for the experimental evaluations in the paper. Here $\eta_s$, $\eta_c$ denote the server and client learning rate and $\eps_s$ refers to the correction term in value in ADAM optimizer of server.}
\label{tab:hyperparam}
\end{table}

\begin{table}[]
\centering
\begin{tabular}{lcccc}
\toprule
\multirow{2}{*}{\textbf{Dataset}} & \multirow{2}{*}{\textbf{Network}} & \multicolumn{3}{c}{\textbf{Federated Averaging (\fedavg)}}                                               \\
                         &                          & \textbf{\acrshort{MQAT}-\small{W}} & \textbf{\acrshort{MQAT}-\small{A}} & \textbf{\acrshort{MQAT}-\small{WA}} \\ 
                         \midrule
\cifar{-10}                 & \slenet{-5}                  & [4,6,8]                & -                          & -                           \\
\cifar{-10}                 & \sresnet{-20}                & [2,3,4,6,8,32]         & [2,3,4,6,8,32]         & [2,3,4,6,8,32]                            \\
\cifar{-100}                & \sresnet{-20}                & [2,3,4,6,8,32]         & [2,3,4,6,8,32]         & -                           \\
\tinyimagenet{}             & \sresnet{-18}                & [2,3,4,6,8]            & -                          & -  \\
\bottomrule
\end{tabular}
\vspace{1ex}
\caption{\em Quantization Bit-Set used to perform \fedavg-\acrshort{MQAT} on different experimental setup in the paper.}
\label{tab:bitset-mqat}
\end{table}

\section{Additional Results}
\label{app:AddResults}
\paragraph{Per-client fixed bit-width.}
Performing a forward-pass in low bit-width during \acrfull{QAT} does have the additional benefit of reduced computational requirements also during \textit{training}. A good assumption to make is that each client implements efficient hardware acceleration for a specific bit-width that remains constant during its participation in the federated learning process. While in the main experimental setting we considered per-round sampling of the bit-width for \acrshort{MQAT}, here we sample a client-specific bit-width at the beginning of training and keep it fixed throughout. 

As we can see in \tabref{tab:mqat_fixbit}, ex-ante fixed bit-widths lead to some degradation in performance, albeit at the aforementioned benefit of spead-up training.

\begin{table}[]
\centering
\begin{tabular}{lccc}
\toprule
\multirow{2}{*}{\textbf{Bit Config}}    & \multicolumn{3}{c}{\textbf{Federated Averaging (\fedavg)}}                                       \\
                               & \textbf{Baseline} & \textbf{\acrshort{MQAT}-\small{W}} & \textbf{\acrshort{MQAT}$^*$-\small{W}} \\ 
\midrule
W-32 & 70.16                                            & \textbf{74.46}                                                      & 70.44                                                      \\
W-8  & 69.86                                            & \textbf{74.54}                                                      & 70.88                                                      \\
W-6  & 70.02                                            & \textbf{74.60}                                                      & 68.24                                                      \\
W-4  & 68.44                                            & \textbf{74.58}                                                      & 70.02                                                      \\
W-3  & 64.14                                            & \textbf{75.64}                                                      & 67.58                                                      \\
W-2  & 28.02                                            & \textbf{72.58}                                                      & 68.20                                                       \\ 
\bottomrule
\end{tabular}
\vspace{1ex}
\caption{\em Global validation accuracy after quantization at various bit-widths for different \fedavg variants trained on \cifar{-10} dataset using \sresnet{-20} architecture. Here, $^*$ indicates the client-specific bit-width is chosen at the begining of training and then kept fixed throughout.} 
\label{tab:mqat_fixbit}
\end{table}

\paragraph{Implicit Regularization of \acrshort{MQAT}.} As discussed in the main paper, \sresnet{-20} on \cifar{-10} exhibits overfitting, which we further investigate here. It was observed in the maint ext that \fedavg-\acrshort{MQAT} has an implicit regularisation effect and thus achieves better validation accuracy for the full-precision model by avoiding overfitting despite that not being the main objective. We show experimental comparisons of our proposed \acrshort{MQAT} to other standard methods of regularization with different strength of weight decay (measured by regularisation term $\lambda_\text{WD}$) and drop-out before the last full-connected layer of the network in \tabref{tab:overfit-cifar}. Despite not being its primary objective, our \fedavg-\acrshort{MQAT} variant achieves considerable gains in full-precision accuracy in comparison to standard approaches used to avoid overfitting. \figref{fig:overfit_comparisons} shows how the validation loss of the baseline is increasing after $2.5k$ rounds and the corresponding validation accuracy on the right. For \acrshort{MQAT}, we observe no overfitting according to the validation loss. 
\begin{figure*}[ht]
     \centering
     \begin{subfigure}[b]{0.45\textwidth}
         \centering
         \includegraphics[width=\textwidth]{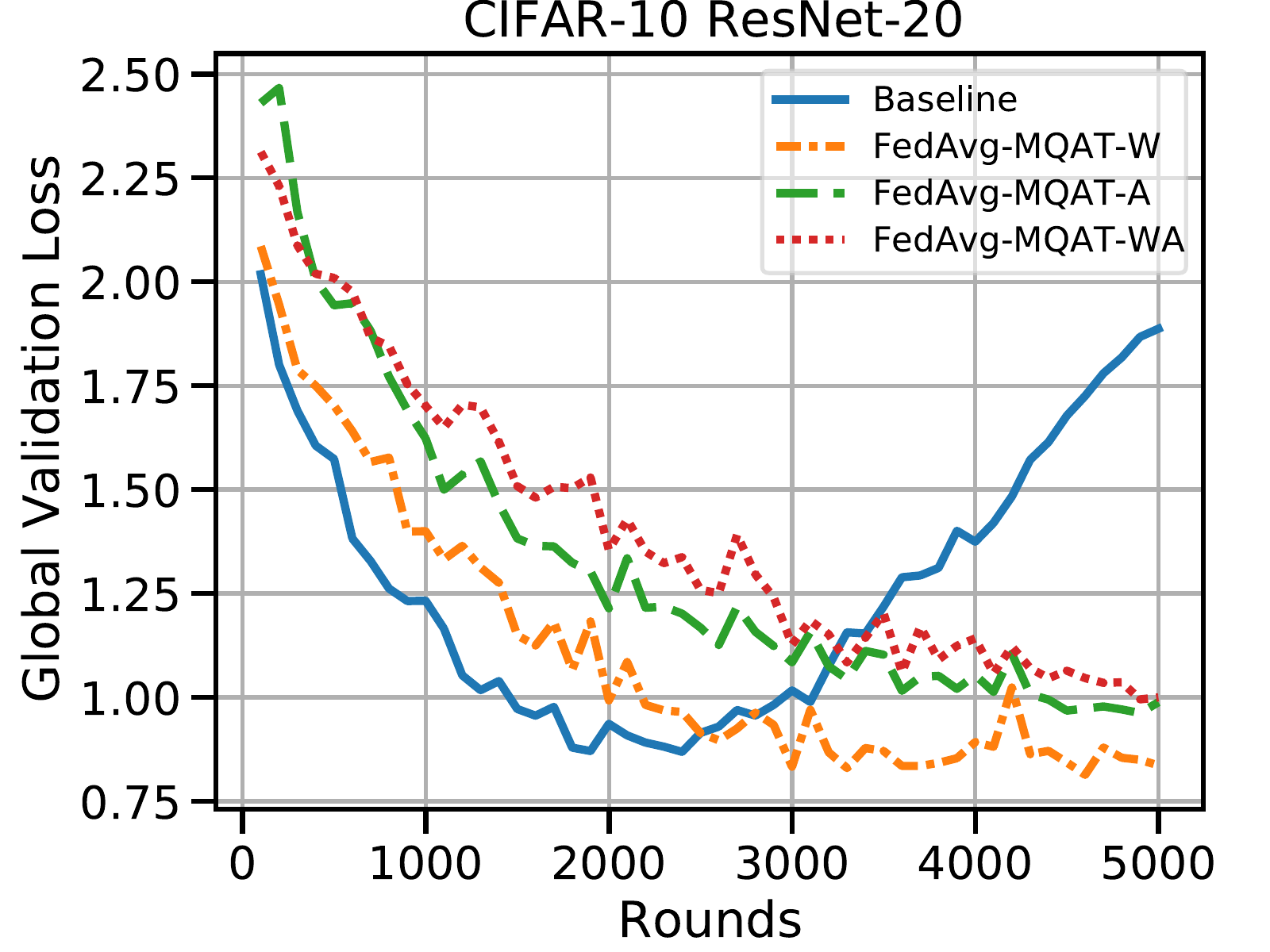}
         \caption{}
         \label{fig:cifar10resnet20_loss}
     \end{subfigure}
     \begin{subfigure}[b]{0.45\textwidth}
         \centering
         \includegraphics[width=\textwidth]{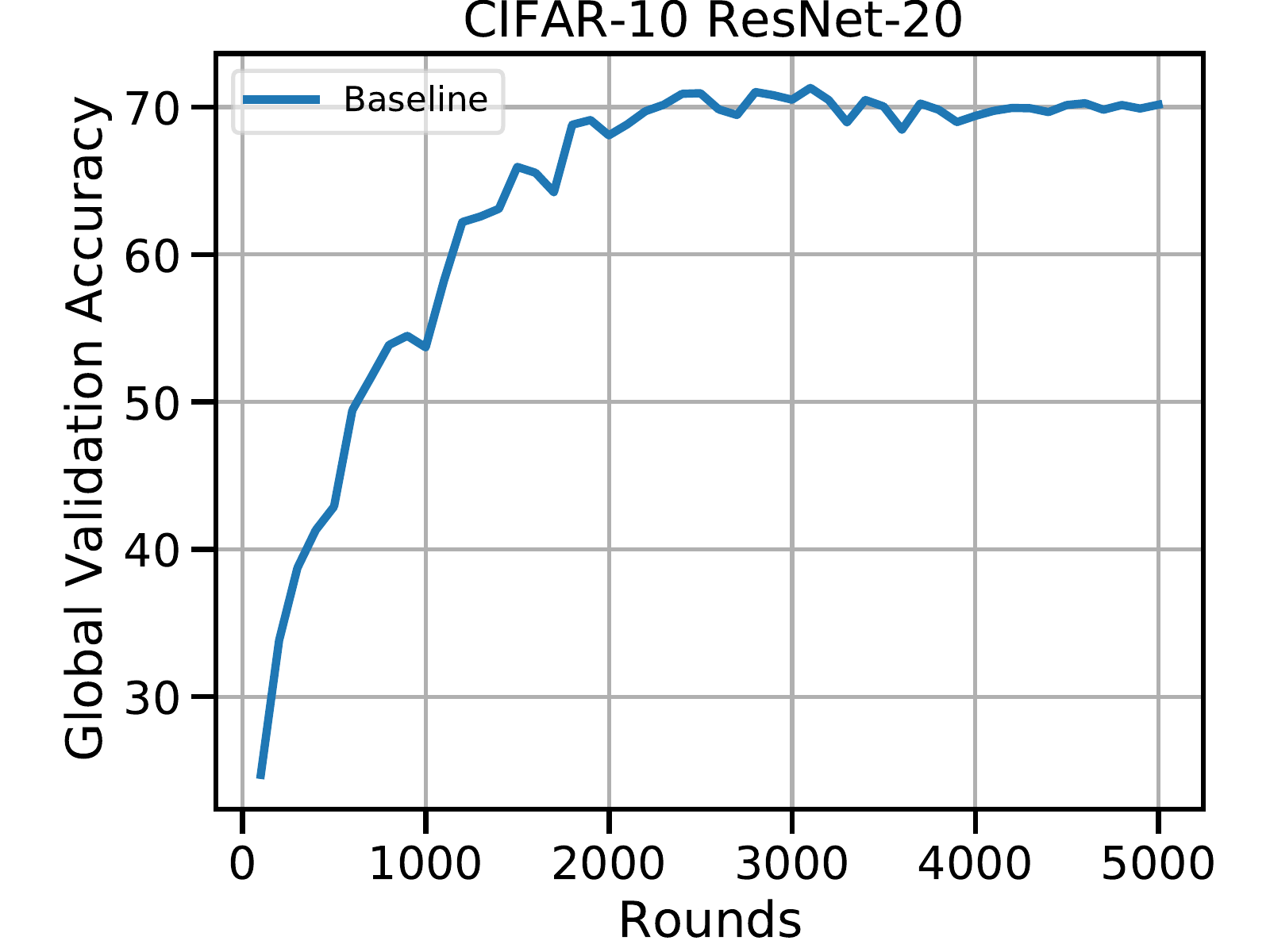}
         \caption{}
         \label{fig:cifar100resnet20_acc}
     \end{subfigure}

        \caption{\em Global validation (a) loss and (b) accuracy curves for Baseline and different \fedavg-\acrshort{MQAT} variants trained on \cifar{-100} using \sresnet{-20} architecture. For \fedavg-\acrshort{MQAT} variants the validation loss refers to the loss after quantization at lowest bit-width in the bit-set $B$. While Baseline model clearly suffers from the overfitting issue, our \fedavg-\acrshort{MQAT} variants clearly manage to avoid it. Further, the validation accuracy curves on the baselines reveals that overfitting cannot be prevented even by early stopping.}
        \label{fig:overfit_comparisons}
\end{figure*}

\paragraph{Detailed results from the main text}
In this section, we provide the exact values used for plotting various figures in the main text in Table 5-11.
\begin{table}[]
\centering
\begin{tabular}{lc}
\toprule
\textbf{\fedavg}              & \textbf{Accuracy} \\
\midrule
Baseline            & 70.16             \\
Baseline (Dropout)  & 72.02             \\
$\lambda_\text{WD} = 1e-2 $ & 44.94             \\
$\lambda_\text{WD} = 1e-3 $ & 71.62             \\
$\lambda_\text{WD} = 1e-4 $ & 70.76             \\
$\lambda_\text{WD} = 1e-5 $ & 70.06             \\
$\lambda_\text{WD} = 1e-6 $ & 70.10              \\
\acrshort{MQAT}-\small{W}              & 74.46             \\
\acrshort{MQAT}-\small{A}              & 76.28             \\
\acrshort{MQAT}-\small{WA}             & 74.90    \\
\bottomrule
\end{tabular}
\vspace{1ex}
\caption{\em Global validation accuracy for full-precision models learnt in federation using \fedavg variants trained on \cifar{-10} dataset with \sresnet{-20} architecture. Here $W$ indicates the client-specific bit-width is chosen at the begining of training and then kept fixed throughout. Here, an abbreviation of ``W'', ``A'' and ``WA'' indicate weight quantization, activation quantization and both weight-activation.} 
\label{tab:overfit-cifar}
\end{table}

\begin{table*}[]
\small
\centering
\begin{tabular}{lccccccc}
\toprule
\multirow{2}{*}{\textbf{Bit Config}}        & \multicolumn{7}{c}{\textbf{Federated Averaging (\fedavg)}} \\

            & \textbf{Baseline} & \textbf{\acrshort{KURE}} & \textbf{\acrshort{APQN} \small{(W-4)}} & \textbf{\acrshort{APQN} \small{(W-2)}} & \textbf{\acrshort{QAT} \small{(W-4)}} & \textbf{\acrshort{QAT} \small{(W-2)}} & \textbf{\acrshort{MQAT}-\small{W}} \\
             \midrule
W-32                        & 70.16                                                                & 69.38                                                                          & 70.5                                                                                                       & 70.04                                                                                                      & 72.38                                                                                                     & 29.06                                                                                                     & 74.46                                                                          \\
W-8                         & 69.86                                                                & 69.38                                                                          & 70.72                                                                                                      & 70.19                                                                                                      & 71.08                                                                                                     & 11.68                                                                                                     & 74.54                                                                          \\
W-6                         & 70.02                                                                & 69.08                                                                          & 70.62                                                                                                      & 70.16                                                                                                      & 71.6                                                                                                      & 12.68                                                                                                     & 74.60                                                                          \\
W-4                         & 68.44                                                                & 68.12                                                                          & 69.44                                                                                                      & 69.36                                                                                                      & 72.58                                                                                                     & 18.5                                                                                                      & 74.58                                                                          \\
W-3                         & 64.14                                                                & 64.28                                                                          & 63.46                                                                                                      & 65.66                                                                                                      & 71.22                                                                                                     & 39.02                                                                                                     & 74.64                                                                          \\
W-2                         & 28.02                                                                & 31.28                                                                          & 29.64                                                                                                      & 33.36                                                                                                      & 52.32                                                                                                     & 72.64                                                                                                     & 72.58                                                                          \\               
\bottomrule 
\end{tabular}
\caption{\em Global validation accuracy after weight quantization of various quantization robustness variants for Federated Averaging (\fedavg) on federated version of \cifar-{10} dataset using \sresnet{-20} architecture. }
\label{tab:cifar10resnet20-wqr}
\end{table*}

\begin{table*}[]
\small
\centering
\begin{tabular}{lccccccc}
\toprule
\multirow{2}{*}{\textbf{Bit Config}}        & \multicolumn{7}{c}{\textbf{Federated Averaging (\fedavg)}} \\

            & \textbf{Baseline} & \textbf{\acrshort{KURE}} & \textbf{\acrshort{APQN} \small{(W-4)}} & \textbf{\acrshort{APQN} \small{(W-2)}} & \textbf{\acrshort{QAT} \small{(W-4)}} & \textbf{\acrshort{QAT} \small{(W-2)}} & \textbf{\acrshort{MQAT}-\small{W}} \\
             \midrule
W-32                        & 49.17                                                                & 49.57                                                                          & 50.77                                                                                                      & 50.8                                                                                                       & 49.19                                                                                                     & 13.62                                                                                                     & 48.47                                                                          \\
W-8                         & 49.24                                                                & 49.52                                                                          & 50.86                                                                                                      & 50.88                                                                                                      & 46.46                                                                                                     & 1.7                                                                                                       & 48.21                                                                          \\
W-6                         & 48.99                                                                & 49.66                                                                          & 50.27                                                                                                      & 50.81                                                                                                      & 47.34                                                                                                     & 1.92                                                                                                      & 48.28                                                                          \\
W-4                         & 45.37                                                                & 47.75                                                                          & 45.76                                                                                                      & 46.34                                                                                                      & 49.7                                                                                                      & 3.05                                                                                                      & 47.18                                                                          \\
W-3                         & 36.42                                                                & 41.04                                                                          & 34.87                                                                                                      & 37.04                                                                                                      & 44.44                                                                                                     & 11.68                                                                                                     & 45.88                                                                          \\
W-2                         & 6.62                                                                 & 8.28                                                                           & 8.41                                                                                                       & 7.62                                                                                                       & 24.77                                                                                                     & 42.92                                                                                                     & 42.27                                                                          \\ 

\bottomrule
\end{tabular}
\caption{\em Global validation accuracy after weight quantization of various quantization robustness variants for Federated Averaging (\fedavg) on federated version of \cifar-{100} dataset using \sresnet{-20} architecture.}
\label{tab:cifar100resnet20-wqr}
\end{table*}
\begin{table*}[]
\small
\centering
\begin{tabular}{lccccccc}
\toprule
\multirow{2}{*}{\textbf{Bit Config}}        & \multicolumn{7}{c}{\textbf{Federated Averaging (\fedavg)}} \\

            & \textbf{Baseline} & \textbf{\acrshort{KURE}} & \textbf{\acrshort{APQN} \small{(W-4)}} & \textbf{\acrshort{APQN} \small{(W-2)}} & \textbf{\acrshort{QAT} \small{(W-4)}} & \textbf{\acrshort{QAT} \small{(W-2)}} & \textbf{\acrshort{MQAT}-\small{W}} \\
             \midrule
W-32                        & 37.21                                                                & 37.51                                                                          & 36.48                                                                                                      & 37.85                                                                                                      & 36.83                                                                                                     & 4.39                                                                                                      & 37.89                                                                          \\
W-8                         & 37.29                                                                & 37.61                                                                          & 36.41                                                                                                      & 37.89                                                                                                      & 36.73                                                                                                     & 2.3                                                                                                       & 37.56                                                                          \\
W-6                         & 36.61                                                                & 37.17                                                                          & 35.9                                                                                                       & 37.35                                                                                                      & 36.98                                                                                                     & 2.73                                                                                                      & 37.23                                                                          \\
W-4                         & 31.09                                                                & 33.29                                                                          & 31.4                                                                                                       & 33.37                                                                                                      & 37.43                                                                                                     & 5.99                                                                                                      & 37.12                                                                          \\
W-3                         & 17.64                                                                & 17.61                                                                          & 20.08                                                                                                      & 19.25                                                                                                      & 31.14                                                                                                     & 15.66                                                                                                     & 36.84                                                                          \\
W-2                         & 0.86                                                                 & 0.9                                                                            & 1.02                                                                                                       & 0.89                                                                                                       & 7.28                                                                                                      & 34.53                                                                                                     & 35.45                                                                          \\ 
\bottomrule
\end{tabular}
\caption{\em Global validation accuracy after weight quantization of various quantization robustness variants for Federated Averaging (\fedavg) on federated version of \tinyimagenet{} dataset using \sresnet{-18} architecture.}
\label{tab:tinyimagenetr18-wqr}
\end{table*}
\begin{table*}[]
\small
\centering
\begin{tabular}{lccccccc}
\toprule
\multirow{2}{*}{\textbf{Bit Config}}        & \multicolumn{7}{c}{\textbf{Federated Averaging (\fedavg)}} \\

            & \textbf{Baseline} & \textbf{\acrshort{KURE}} & \textbf{\acrshort{APQN} \small{(W-4)}} & \textbf{\acrshort{APQN} \small{(W-2)}} & \textbf{\acrshort{QAT} \small{(W-4)}} & \textbf{\acrshort{QAT} \small{(W-2)}} & \textbf{\acrshort{MQAT}-\small{W}} \\
             \midrule
W-32                        & 69                                                                   & 69.4                                                                           & 69.66                                                                                                      & 68.92                                                                                                      & 65.68                                                                                                     & 27.74                                                                                                     & 65.68                                                                          \\
W-8                         & 69.02                                                                & 69.4                                                                           & 69.66                                                                                                      & 68.86                                                                                                      & 55.1                                                                                                      & 11.52                                                                                                     & 65.66                                                                          \\
W-6                         & 68.72                                                                & 69.38                                                                          & 69.54                                                                                                      & 68.54                                                                                                      & 59.42                                                                                                     & 11.74                                                                                                     & 66.26                                                                          \\
W-4                         & 68.24                                                                & 68.8                                                                           & 68.66                                                                                                      & 68.26                                                                                                      & 66.74                                                                                                     & 12.62                                                                                                     & 66.02                                                                          \\
W-3                         & 66.82                                                                & 67.12                                                                          & 67.06                                                                                                      & 66.94                                                                                                      & 58.4                                                                                                      & 15.36                                                                                                     & 62.51                                                                          \\
W-2                         & 48.68                                                                & 51.94                                                                          & 51.3                                                                                                       & 53.14                                                                                                      & 61.44                                                                                                     & 61.5                                                                                                      & 59.32                                                                          \\ 

\bottomrule
\end{tabular}
\caption{\em Global validation accuracy after weight quantization of various quantization robustness variants for Federated Averaging (\fedavg) on federated version of \cifar-{10} dataset using \slenet{-5} architecture.} 
\label{tab:cifar10lenet5-wqr}
\end{table*}

\begin{table*}[]
\small
\centering
\begin{tabular}{lccccccc}
\toprule
\multirow{2}{*}{\textbf{Bit Config}}        & \multicolumn{7}{c}{\textbf{Federated Averaging (\fedavg)}} \\

            & \textbf{Baseline} & \textbf{\acrshort{KURE}} & \textbf{\acrshort{APQN} \small{(A-4)}} & \textbf{\acrshort{APQN} \small{(A-2)}} & \textbf{\acrshort{QAT} \small{(A-4)}} & \textbf{\acrshort{QAT} \small{(A-2)}} & \textbf{\acrshort{MQAT}-\small{A}} \\
             \midrule
A-32                        & 70.16                                                                & 66.82                                                                          & 70.42                                                                                                      & 50.56                                                                                                      & 63.74                                                                                                     & 41.04                                                                                                     & 76.28                                                                          \\
A-8                         & 70.06                                                                & 66.88                                                                          & 70.3                                                                                                       & 50.55                                                                                                      & 73.98                                                                                                     & 53.24                                                                                                     & 76.38                                                                          \\
A-6                         & 70                                                                   & 66.78                                                                          & 70.1                                                                                                       & 49.98                                                                                                      & 73.72                                                                                                     & 53.22                                                                                                     & 76.42                                                                          \\
A-4                         & 64.5                                                                 & 63.38                                                                          & 65.66                                                                                                      & 41.49                                                                                                      & 71.26                                                                                                     & 53.1                                                                                                      & 76.14                                                                          \\
A-3                         & 47.92                                                                & 50.03                                                                          & 51.48                                                                                                      & 22.72                                                                                                      & 51.52                                                                                                     & 53.1                                                                                                      & 74.68                                                                           \\
A-2                         & 12.66                                                                & 26.2                                                                           & 15.2                                                                                                       & 3.8                                                                                                        & 11.06                                                                                                     & 58.2                                                                                                      & 67.78                                                                          \\ 
\bottomrule 
\end{tabular}
\caption{\em Global validation accuracy after activation quantization of various quantization robustness variants for Federated Averaging (\fedavg) on federated version of \cifar-{10} dataset using \sresnet{-20} architecture. }
\label{tab:cifar10resnet20-aqr}
\end{table*}

\begin{table*}[]
\small
\centering
\begin{tabular}{lccccccc}
\toprule
\multirow{2}{*}{\textbf{Bit Config}}        & \multicolumn{7}{c}{\textbf{Federated Averaging (\fedavg)}} \\

            & \textbf{Baseline} & \textbf{\acrshort{KURE}} & \textbf{\acrshort{APQN} \small{(A-4)}} & \textbf{\acrshort{APQN} \small{(A-2)}} & \textbf{\acrshort{QAT} \small{(A-4)}} & \textbf{\acrshort{QAT} \small{(A-2)}} & \textbf{\acrshort{MQAT}-\small{A}} \\
             \midrule
A-32                        & 49.17                                                                & 37.36                                                                          & 50.94                                                                                                      & 50.56                                                                                                      & 18.13                                                                                                     & 4.74                                                                                                      & 40.71                                                                          \\
A-8                         & 49.37                                                                & 37.22                                                                          & 50.86                                                                                                      & 50.55                                                                                                      & 45.62                                                                                                     & 23.91                                                                                                     & 43.79                                                                          \\
A-6                         & 48.73                                                                & 37.25                                                                          & 50                                                                                                         & 49.98                                                                                                      & 45.57                                                                                                     & 23.82                                                                                                     & 44.23                                                                          \\
A-4                         & 41.36                                                                & 32.5                                                                           & 40.6                                                                                                       & 41.4                                                                                                       & 43.06                                                                                                     & 23.13                                                                                                     & 42.59                                                                          \\
A-3                         & 24.33                                                                & 20.16                                                                          & 22.22                                                                                                      & 22.72                                                                                                      & 25.31                                                                                                     & 24.14                                                                                                     & 30.5                                                                           \\
A-2                         & 4.59                                                                 & 3.4                                                                            & 4.13                                                                                                       & 3.8                                                                                                        & 1.24                                                                                                      & 7.81                                                                                                      & 28.79                                                                          \\ 
\bottomrule 
\end{tabular}
\caption{\em Global validation accuracy after activation quantization of various quantization robustness variants for Federated Averaging (\fedavg) on federated version of \cifar-{100} dataset using \sresnet{-20} architecture. }
\label{tab:cifar100resnet20-aqr}
\end{table*}

\begin{table*}[]
\small
\centering
\begin{tabular}{lccccc}
\toprule
\multirow{2}{*}{\textbf{Bit Config}}        & \multicolumn{5}{c}{\textbf{Federated Averaging (\fedavg)}} \\

            & \textbf{Baseline} & \textbf{Baseline (Dropout)} & \textbf{\acrshort{QAT} \small{(WA-4)}} & \textbf{\acrshort{QAT} \small{(WA-2)}} & \textbf{\acrshort{MQAT}-\small{WA}} \\
             \midrule
WA-32/32                    & 70.16                                                                & 72.02                                                                          & 60.64                                                                                                     & 26.61                                                                                                     & 74.9                                                                           \\
WA-8/8                      & 69.86                                                                & 71.62                                                                          & 71                                                                                                        & 28.08                                                                                                     & 77.3                                                                           \\
WA-6/6                      & 69.92                                                                & 70.58                                                                          & 71.38                                                                                                     & 29.42                                                                                                     & 76.8                                                                           \\
WA-4/4                      & 63.12                                                                & 64.08                                                                          & 70.04                                                                                                     & 36.44                                                                                                     & 75.54                                                                          \\
WA-3/3                      & 42.98                                                                & 41.8                                                                           & 42.54                                                                                                     & 50.12                                                                                                     & 72.9                                                                           \\
WA-2/2                      & 11.04                                                                & 14.5                                                                           & 9.58                                                                                                      & 40.18                                                                                                     & 66.08                                                                          \\ 
\bottomrule 
\end{tabular}
\caption{\em Global validation accuracy after activation quantization of various quantization robustness variants for Federated Averaging (\fedavg) on federated version of \cifar-{10} dataset using \sresnet{-20} architecture. }
\label{tab:cifar10resnet20-aqr}
\end{table*}


\end{document}